\documentclass{article}

\usepackage{microtype}
\usepackage{graphicx}
\usepackage{subcaption}
\usepackage{booktabs} % for professional tables

\usepackage{hyperref}

\usepackage[preprint]{icml2026}

\usepackage{amsmath}
\usepackage{amssymb}
\usepackage{mathtools}
\usepackage{amsthm}

\usepackage[capitalize,noabbrev]{cleveref}

\newcommand{\mc}{\mathcal}
\newcommand{\bm}{\mathbf}
\newcommand{\N}{\mathbb{N}}
\newcommand{\mR}{\mathcal{\mathbf{R}}}
\newcommand{\B}{\mathbb{B}}
\newcommand{\vx}{\mathbf{x}}
\newcommand{\vu}{\mathbf{u}}
\newcommand{\vw}{\mathbf{w}}
\newcommand{\vy}{\mathbf{y}}

\newcommand{\mX}{\mathbf{X}}
\newcommand{\E}{\mathbb{E}}
\newcommand{\1}{\mathbf{1}}
\DeclareMathOperator{\bayes}{Bayes}

\renewcommand{\P}{\mathbb{P}}
\newcommand{\R}{\mathbb{R}}

\DeclareMathOperator{\poly}{Poly}
\DeclareMathOperator{\diam}{diam}

\title{Harmful Overfitting in Sobolev Spaces}

\theoremstyle{plain}
\newtheorem{theorem}{Theorem}[section]

\newtheorem{lemma}[theorem]{Lemma}
\newtheorem{corollary}[theorem]{Corollary}
\theoremstyle{definition}

\newtheorem{assumption}[theorem]{Assumption}
\theoremstyle{remark}

\usepackage[textsize=tiny]{todonotes}

\icmltitlerunning{Harmful Overfitting in Sobolev Spaces}

\begin{document}

\twocolumn[
  \icmltitle{Harmful Overfitting in Sobolev Spaces}



  \icmlsetsymbol{equal}{*}

  \begin{icmlauthorlist}
    \icmlauthor{Kedar Karhadkar}{equal,ucla}
    \icmlauthor{Alexander Sietsema}{equal,ucla}
    \icmlauthor{Deanna Needell}{ucla}
    \icmlauthor{Guido Montufar}{ucla,stat}
  \end{icmlauthorlist}

  \icmlaffiliation{ucla}{Department of Mathematics, University of California, Los Angeles, Los Angeles, CA, USA}
  
  \icmlaffiliation{stat}{Department of Statistics \& Data Science, University of California, Los Angeles, Los Angeles, CA, USA}

  \icmlcorrespondingauthor{Kedar Karhadkar}{kedar@math.ucla.edu}

  \icmlkeywords{Machine Learning, ICML}

  \vskip 0.3in
]



\printAffiliationsAndNotice{\icmlEqualContribution}

\begin{abstract}
  Motivated by recent work on benign overfitting in overparameterized machine learning, we study the generalization behavior of functions in Sobolev spaces $W^{k, p}(\R^d)$ that perfectly fit a noisy training data set.
  Under assumptions of label noise and sufficient regularity in the data distribution, we show that approximately norm-minimizing interpolators, which are canonical solutions selected by smoothness bias, exhibit harmful overfitting:
  even as the training sample size $n \to \infty$, the generalization error remains bounded below by a positive constant with high probability.
  Our results hold for arbitrary values of $p \in [1, \infty)$, in contrast to prior results studying the Hilbert space case ($p = 2$) using kernel methods. Our proof uses a geometric argument which identifies harmful neighborhoods of the training data using Sobolev inequalities.
\end{abstract}

\section{Introduction}

A classical tenet of statistical learning theory is that exact interpolation of noisy data should lead to poor generalization. Surprisingly, a growing body of recent work has shown that this intuition can fail in modern overparameterized regimes: 
in certain settings, a statistical model can perfectly fit a noisy training dataset, potentially with several or even many incorrect labels, while still generalizing well to unseen test data, a phenomenon now referred to as benign overfitting.

The ability of a model to exhibit benign overfitting is particularly relevant because standard learning algorithms fit noisy training data via empirical risk minimization, and its occurrence indicates a surprising robustness of the learned model to label noise. Consequently, there is significant interest in understanding when benign overfitting can arise and when it necessarily fails.
There has been some work on benign overfitting in neural networks but almost all of it is for extremely high-dimensional data, where the dimension must significantly exceed the number of data points. 
Here, we instead aim to understand the behavior of fixed-dimension data which is often more realistic. 

\subsection*{Our contributions}

Our main result (Theorem~\ref{thm:tempered-overfitting-bayes-v2}) shows approximately norm-minimizing interpolators in Sobolev spaces \textit{cannot} benignly overfit. Under standard assumptions on label noise and mild regularity of the data distribution, we prove that any interpolating solution whose Sobolev norm is within a constant factor of the minimum necessarily suffers from persistent excess risk. In particular, even as the number of training samples tends to infinity, the population error of such solutions remains bounded away from the Bayes-optimal risk with high probability. This is a significant generalization of prior work, which works with either specific model classes within Sobolev spaces or with general Sobolev spaces but more restricted parameter bounds.

More precisely, we establish a uniform lower bound on the expected excess risk of all approximately norm-minimizing interpolators in $W^{k,p}(\mathbb{R}^d)$ for a broad range of smoothness parameters, including values beyond the Hilbert space case $p=2$. 
This bound depends only on the Sobolev parameters, the data distribution, and the noise level, but is independent of the sample size, as long as the sample size is large enough.
This demonstrates that smoothness bias alone does not guarantee benign overfitting in Sobolev spaces, and that norm minimization can in fact lead to \emph{harmful overfitting}.

Our results improve upon existing work on interpolation in Sobolev spaces and kernel regression by establishing that harmful overfitting occurs in a more general setting than previously studied. In particular, we loosen the following requirements:
\begin{itemize}
    \item We consider interpolators in spaces $W^{k, p}(\R^d)$ for $p \in [1, \infty)$, which goes beyond the commonly studied case $p = 2$. For $p \neq 2$, $W^{k, p}(\R^d)$ is not a Hilbert space, and minimum-norm interpolation is no longer a linear problem.
    \item We study \emph{approximately} norm-minimizing (ANM) interpolators rather than just norm-minimizing interpolators. ANM functions are defined by having sufficient smoothness as determined by the Sobolev norm and do not take a particular functional form.
    \item We consider a general class of data distributions with label noise which in particular includes regression with broad classes of additive heteroskedastic noise. 
    \item We assume that the loss function satisfies a mild growth condition, which includes all $\ell^q$ losses for $q \in [1, \infty)$. 
\end{itemize}

\medskip 
\textbf{Organization:} This paper is organized as follows: 
after reviewing related work in Section~\ref{sec:related-work}, 
in Section~\ref{sec:setup}, we outline our assumptions and state the main theorem. 
In Section~\ref{sec:proof-overview}, we provide an overview of the proof of the main theorem including the statements of several intermediate results.
Finally, in Section~\ref{sec:main-result-proof}, we provide a complete proof of the main theorem.
All proofs from Sections \ref{sec:setup} and \ref{sec:proof-overview} are provided in the appendices.

\section{Related work}
\label{sec:related-work}

The phenomenon of \emph{benign overfitting}, where models that interpolate noisy training data nonetheless achieve strong generalization, has attracted significant attention in recent years. 
Early theoretical work established that exact interpolation does not necessarily preclude consistency, challenging classical statistical intuitions \cite{bartlett2020benign, zhang2021understanding}. Building on this line of research, \citet{mallinar2022benign} proposed a taxonomy distinguishing \emph{benign}, \emph{tempered}, and \emph{catastrophic} overfitting, providing a unified framework for comparing different interpolation regimes. 

Benign overfitting has been most thoroughly understood in linear settings. In particular, \citet{bartlett2020benign} and subsequently \citet{9051968, pmlr-v134-zou21a, 10.1214/21-AOS2133, koehler2021uniform, JMLR:v23:21-1199, pmlr-v178-shamir22a, wang2022binary} analyzed minimum-norm interpolation in overparameterized linear regression, identifying precise conditions under which interpolators generalize despite fitting noisy labels. These results highlight the central role of the inductive bias induced by norm minimization. 

More recently, a growing body of work has investigated benign overfitting in nonlinear models, including neural networks and transformer architectures. Several studies provide theoretical or empirical evidence that overparameterized neural networks trained by gradient-based methods on high-dimensional data can exhibit benign overfitting under suitable conditions \citep{cao2022benign, FreiVBS23, george2023training, xu2024benign, karhadkar2024benign, magen2024benign,kou2023benign}. Together, these works suggest that benign overfitting is not restricted to linear or kernelized models, but may arise more broadly in modern deep learning architectures.

Related to this, a number of papers have explored \emph{tempered overfitting}, an intermediate regime in which the test error does not vanish but remains controlled. In the setting of threshold networks with binary weights, \citet{harel2024provable} establish provable guarantees for tempered overfitting. The role of ambient dimension is further investigated by \citet{kornowski2023tempered}, who show that two-layer neural networks trained on constant data with label-flipping noise exhibit benign overfitting in very high dimensions ($d \gtrsim n^2 \log n$), but only tempered overfitting in one dimension. 
The impact of the loss function has also been highlighted: \citet{joshi2023noisy} demonstrate that, for two-layer minimum-norm univariate networks interpolating noisy data, the $L^1$ loss leads to tempered overfitting, whereas the $L^2$ loss results in catastrophic overfitting.

Beyond gradient-based learning in parametric models, other learning rules have also been analyzed through the lens of interpolation. \citet{barzilai2025beyond} show that Nadaraya-Watson interpolators can exhibit benign, tempered, or catastrophic overfitting depending on the choice of the locality parameter. 
Classical results on nearest-neighbor methods provide an early example of interpolation with controlled risk: \citet{cover1967nearest} showed that the 1-nearest neighbor classifier has risk at most twice that of the Bayes-optimal predictor. \citet{kur2024minimum} prove an upper bound for the generalization error for minimum-norm interpolators in Banach spaces based on Rademacher complexity and the geometry of the space which is sharp for $\ell^p$ linear regression $(p \in [1, 2])$.

Finally, several works have studied harmful overfitting in kernel regression. Of particular relevance to the present work, \citet{rakhlin2019consistency} show that minimum-norm kernel regression exhibits harmful overfitting when the associated reproducing kernel Hilbert space (RKHS) is a Sobolev space $W^{k,2}(\mathbb{R}^d)$ with $d$ odd and $k = (d + 1)/2$. \citet{buchholz2022kernel} generalize upon this work by showing that kernel regression exhibits harmful overfitting when the RKHS is a Sobolev space $W^{k, p}(\R^d)$ with $d/2 < k < 3d/4$. \citet{beaglehole2023inconsistency} establish that minimum-norm kernel regression is inconsistent for a broader set of kernels, and investigate the effect of the bandwidth of the kernel on generalization error. \citet{pmlr-v235-cheng24g} identify regimes in which kernel regression is benign, tempered, and catastrophic in terms of the spectrum of the kernel. \citet{li2024kernel} show that kernel regression generalizes poorly for a class of kernels satisfying a particular eigenvalue decay condition.

In contrast, \citet{haas2023mind} demonstrate that interpolation can be consistent for kernel regression when the kernel bandwidth is chosen appropriately. 
 These results underscore that interpolation behavior in Sobolev-type function classes is delicate and is sensitive to both regularity and inductive bias.

\section{Setup and main result}
\label{sec:setup}

In this work, we consider a purely probabilistic data model, which includes a variety of regression problems as special cases. 
We let $\Omega$ be a bounded open subset of $\R^d$ with $C^1$ boundary, and let $\mu$ be a Borel probability measure on $\overline{\Omega} \times \R$. 
We denote the marginal distributions of $\mu$ in the first and second coordinates by $\mu_{\vx}$ and $\mu_y$, respectively. For $i \in [n]$, we sample training pairs $(\vx_i, y_i)$ i.i.d.\ from $\mu$, where $x_i \in \R^d$ represents an input data point and $y_i \in \R$ represents the corresponding label. We also let $\ell: \R \times \R \to [0, \infty)$ be a continuous loss function with $\ell(\hat{y}, y) = 0$ if and only if $\hat{y} = y$. 

When studying overfitting, a natural type of predictor to consider is one which perfectly fits a training dataset while being ``as simple as possible" according to some notion of complexity. We will interpret ``simple" to mean a function which is smooth in the sense that its derivatives take small values. We can formalize this through the notion of Sobolev spaces. The Sobolev space $W^{k, p}(\R^d)$ consists of functions $g: \R^d \to \R$ whose derivatives up to order $k$ have finite $L^p$-norm. We define the norm in $W^{k, p}(\R^d)$ by
\[\|g\|_{W^{k, p}(\R^d)} = \sum_{|\alpha| \leq k}\left(\int_{\R^d} |D^{\alpha}g(\vx)|^p d\vx \right)^{1/p},  \]
where $D^{\alpha} g$ denotes a higher order partial derivative of $g$ multi-indexed by $\alpha = (\alpha_1,\ldots, \alpha_d)$. 
We will use this norm as our measure of the complexity of a function.

We consider functions that are elements of a Sobolev space $W^{k, p}(\R^d)$ where $kp > d$, so pointwise evaluation is well-defined and continuous by the Sobolev embedding theorem (see Theorem~\ref{thm:sobolev-embedding}). 
We say that $f^*$ is a \emph{minimum-norm solution} for the dataset if it solves the optimization problem
\begin{align}
\begin{aligned}
    &\min_{f \in W^{k, p}(\R^d)} \| f\|_{W^{k, p}(\R^d)}\\
    &\text{subject to } \ell(f(\vx_i), y_i) = 0 \text{ for all } i \in [n].
\end{aligned}
    \label{eqn:min-norm-sobolev}
\end{align}
We consider a more general class of solutions. We say that $f \in W^{k, p}(\R^d)$ is an \emph{approximately norm minimizing solution with factor $\gamma \geq 1$} ($\gamma$-ANM solution) if $\ell(f(\vx_i), y_i) = 0$ for all $i \in [n]$ and $\|f\|_{W^{k, p}} \leq \gamma \|f^*\|_{W^{k, p}}$, where $f^*$ is a minimum-norm solution.

\textbf{Notation:} We fix the following notation: if $X$ is a metric space, and $x \in X$, we write $B(x, r)$ to denote the open ball of radius $r$ about $x$. If $E \subset \R^d$ is measurable, we denote its Lebesgue measure by $|E|$. If $X$ is a metric space, we let $\mc{B}(X)$ denote the set of Borel subsets of $X$. 
We use $\poly_{S}(r)$ to denote an arbitrary polynomial in $r$ whose coefficients and degree may depend on the elements of some set $S$, and use $\gtrsim_S, \lesssim_S$ to denote inequality up to some constant dependent only on the elements of $S$.

\subsection{Assumptions on loss function and data distribution} 
\label{sec:assumptions}

For our main results, we impose some additional assumptions on the loss function and data distribution.

\begin{assumption}[Growth rate of loss function]\label{assump:growth-rate-loss}
    There exist constants $C_{\ell}, \tau_{\ell} > 0$ such that for all $\hat{y}, y \in \R$,
    \[\ell(\hat{y}, y) \leq C_{\ell}\exp(\tau_{\ell}(1 + |\hat{y}|)(1 + |y|)). \]
\end{assumption}

The above assumption is satisfied by most losses in regression settings, including $\ell^q$ losses for $q \in [1, \infty)$. We impose this mild growth condition to prevent outliers in the output dataset from having an extremely large effect on the population loss. One could relax this assumption by imposing stricter regularity on the range of values possible for the output distribution (such as boundedness).

Next, we assume that the conditional distribution of $y$ given $\vx$ exists and is sufficiently regular.

\begin{assumption}[Regularity of conditional distribution]\label{assump:regularity-conditional-distribution}
    For $(\vx, y) \sim \mu$, there exists a regular conditional probability $(\vx_0, A) \mapsto \P(y \in A \mid \vx = \vx_0)$. There also exists a version of this regular conditional probability $\nu: \overline{\Omega} \times \mc{B}(\R) \to [0, 1]$ which is weakly continuous in the following sense. If $\vx_m \to \vx_0$ and $g: \R \to \R$ is continuous and bounded, then
    \begin{align*}
        \lim_{m \to \infty} \int_{\R} g d \nu(\vx_m, \bullet) &= \int_{\R} g  d\nu(\vx_0, \bullet).  
    \end{align*}
    We often write $\nu_{\vx_0}$ to denote the probability measure $\nu(\vx_0, \bullet)$.
\end{assumption}
We define the \emph{conditional loss} $\mc{L}: \R \times \overline{\Omega} \to [0, \infty]$ by
\begin{align*}
    \mc{L}(\hat{y}; \vx) = \int_{\R} \ell(\hat{y}, y) d \nu_{\vx}(y).
\end{align*}
The conditional loss measures the expected loss given that the input data point is $\vx$ and we predict the value $\hat{y}$. Assumption~\ref{assump:regularity-conditional-distribution} is the minimal one needed to ensure that the conditional loss exists and is well-behaved. It is in particular satisfied whenever the distribution of $(\vx, y)$ has a continuous  bounded density function.

Next, we will assume that the marginal distribution of $\vx$ is sufficiently regular and the density does not attain arbitrarily large values.
\begin{assumption}[Regularity of marginal distribution]\label{assump:data-regularity}
    There exist constants $c_{\mc{D}}, C_{\mc{D}} > 0$ such that $\mu_{\vx}$ has a density $p_{\vx}: \overline{\Omega} \to [0, \infty)$ satisfying
    \[c_{\mc{D}} \leq p_{\vx}(\vx_0) \leq C_{\mc{D}}\]
    for all $\vx_0 \in \overline{\Omega}$.
\end{assumption}
We impose this condition because our proof is fundamentally geometric. For each training point which has high conditional loss, we find a neighborhood of points which also has high conditional loss. This region has high Euclidean volume, and if the density function is well-behaved, it also has high probability.

We also assume that our output label distribution is sufficiently weak-tailed.

\begin{assumption}[Conditionally sub-Gaussian output]
\label{assump:subgaussian-outputs}
    There exists a constant $C_y > 0$ such that for $(\vx, y) \sim \mu$ and all $t \geq 0$,
    \begin{align*}
        \P\left(|y| \geq t \mid \vx \right) \leq 2\exp\left(-\frac{t^2}{C_y^2} \right)
    \end{align*}
    almost surely.
\end{assumption}
This assumption complements Assumption~\ref{assump:growth-rate-loss} in ensuring that outliers do not affect the population loss pathologically.

The following assumption encodes that predicting using only a single (constant) output label will always be suboptimal.
\begin{assumption}[Label noise]\label{assump:label-noise}
    Suppose that $(\vx, y) \sim \mu$. Then there exist universal constants $\sigma > 0$, $\rho \in (0, 1]$ such that
    \begin{align*}
        \P \left( \mc{L}(y; \vx) \geq \sigma + \inf_{\hat{y} \in \R} \mc{L}(\hat{y}; \vx)  \middle| \vx \right) \geq \rho
    \end{align*}
    almost surely.
\end{assumption}
This condition is analogous in purpose to the assumption that $\text{Var}(y \mid \vx) \geq \sigma^2$ in that it is requiring there to be some probability of a point being ``mislabeled". 

We say that a Borel-measurable function $f: \overline{\Omega} \to \R$ is \emph{Bayes-optimal} if for $(\vx, y) \sim \mu$,
\begin{align*}
    \E[\ell(f(\vx), y) \mid \vx] = \inf_{\hat{y} \in \R} \E[\ell(\hat{y}, y) \mid \vx]
\end{align*}
almost surely.
\begin{assumption}[Regularity of Bayes-optimal function]\label{assump:bayes-optimal-regularity}
    There exists a Bayes-optimal function $f_{\bayes} \in W^{k, p}(\R^d)$.
\end{assumption}
We can interpret the above assumption as saying that although the data distribution is noisy, there is some de-noised ground truth which is sufficiently smooth.

\subsection{Main theorem}
\label{sec:main-theorem}

With these assumptions in place, we now state our main result.
This theorem establishes that, with high probability, the expected regret of any $\gamma$-ANM interpolant compared to the Bayes optimizer is bounded below by a constant independent of $n$. In other words, we show that any approximately norm-minimizing interpolant has at least constant generalization error even as $n \to \infty$, so we cannot benignly overfit.
\begin{theorem}
\label{thm:tempered-overfitting-bayes-v2}
    Let $\epsilon \in (0,1)$, let $k \in (d/p, 1.5d/p)$, and let $n \gtrsim \rho^{-2} \log(\epsilon^{-1}) + \poly_{k, p, d}(\epsilon^{-1})$. Let $(\vx, y) \sim \mu$ be a test point chosen independently from the training set $(\mX, \vy)$. Then with probability at least $1-\epsilon$ over the training set, the following holds. For all $\gamma$-ANM solutions $f_{\gamma}$:
    $$
    \E[\ell(f_{\gamma}(\vx), y) - \ell(f_{\bayes}(\vx), y) \mid \mX, \vy] \geq C\gamma^{-pd/(kp - d)},
    $$
    where $C \in (0, \infty)$ is a constant depending on $k, d, p, \mu$, and $\ell$. 
\end{theorem}

To demonstrate this theorem in a %more 
concrete setting, we apply it to the case of Gaussian heteroskedastic noise using the squared loss $\ell(\hat{y}, y) = (\hat{y} - y)^2$. Given a ground truth function $g \in W^{k,p}(\Omega)$, % by Sobolev embedding, 
suppose that
$$
y = g(\vx) + \epsilon,
$$
where $\epsilon$ is drawn from a Gaussian distribution $\mc{N}(0, \sigma^2(\vx))$ conditional on $\vx$, with $0 < \sigma_{\min} \leq \sigma(\vx) \leq \sigma_{\max}$ for some constants $\sigma_{\min}, \sigma_{\max} > 0$, and $\vx \mapsto \sigma(\vx)$ is continuous. Additionally, suppose that the distribution $\mc{D}_{\vx}$ has a density function $p_{\vx}$ satisfying Assumption~\ref{assump:data-regularity}.
By checking that the rest of the assumptions of Theorem~\ref{thm:tempered-overfitting-bayes-v2} hold under these conditions, we again conclude that the generalization error of any $\gamma$-ANM interpolator is bounded below; in this case, the generalization error simplifies to the $L^2(\mu)$ error between the learned solution and $g$.

\begin{corollary}
\label{corr:square-loss}
    Let $\epsilon \in (0,1)$, let $k \in (d/p, 1.5d/p)$, and let $n \gtrsim  \log(\epsilon^{-1}) + \poly_{k, p, d}(\epsilon^{-1})$. Then with probability at least $1-\epsilon$ over the training set, the following holds. For all $\gamma$-ANM solutions $f_{\gamma}$:
    $$
    \|f_\gamma - g\|_{L^2(\mu)}^2 \geq C\gamma^{-pd/(kp - d)},
    $$
    where $C \in (0, \infty)$ is a constant depending on $k, d, p$, and $\mu$. 
\end{corollary}

\section{Proof overview and supporting results}
\label{sec:proof-overview}

Our proof follows three main steps:

\begin{enumerate}
    \item Explicitly identify an interpolating solution and give an upper bound on its norm. This bounds the norm of any minimum-norm solution.
    \item Show that there are a large number of points in the dataset which are sufficiently noisy and sufficiently separated from the other data points.
    \item Show that around these points, any $\gamma$-ANM solution must be smooth enough not to violate the minimum-norm bound, and thus accumulates generalization error around these points.
\end{enumerate}

In this section, we will expand on each of these steps.

\subsection{Explicit interpolating solution}
\label{sec:interpolating-solution}

Our explicit interpolating function is based on \textit{bump functions}, which are continuous, radially symmetric functions with compact support. Bump functions and radial basis functions are common tools for analyzing the local behavior of Sobolev functions \citep{evans2022partial, adams2003sobolev}.
We first show the existence of bump functions supported on balls of arbitrary radius with bounded norm. 

\begin{lemma}\label{lem:bump-function}
    Let $\vx_0 \in \R^d$ and let $\delta > 0$. There exists $\psi \in W^{k, p}(\R^d)$  satisfying the following properties:
    \begin{enumerate}
        \item For all $\vx \in \R^d$, $\psi(\vx) \in [0, 1]$.
        \item For all $\vx \in B(\vx_0, \delta /2 )$, $\psi(\vx) = 1$.
        \item For all $\vx \notin B(\vx_0, \delta)$, $\psi(\vx) = 0$.
        \item $\|\psi\|_{W^{k, p}(\R^d)} \lesssim_{k, d, p} 1 + \delta^{(d - kp)/p}$.
    \end{enumerate}
\end{lemma}

We construct an interpolant by placing bump functions around each data point in the dataset with magnitudes corresponding to the associated label. To ensure that these functions have non-overlapping supports, we show that the corresponding balls are disjoint. For a dataset $\mX \in \R^{n \times d}$, we define the values $\delta_1(\mX), \cdots, \delta_n(\mX) \in (0, \infty)$ as follows. For $i \in [n]$, let $\delta_i(\mX)$ be the radius of the largest ball about $\vx_i$ not containing any other data points. 
That is,
\[
\delta_i(\mX) = \min_{\ell \neq i}\|\vx_i - \vx_{\ell}\|. 
\]
We write $\delta_i$ in place of $\delta_i(\mX)$, except in situations where it is necessary to specify the dataset.
With this definition, we establish that the radius-$\delta_i/2$ balls are disjoint. 

\begin{lemma}\label{lem:delta-packing}
    The sets $B(\vx_i, \delta_i/2)$ for $i \in [n]$ are disjoint.
\end{lemma}

We then define our interpolant for the dataset $(\mX, \vy)$ as 
\[
f = \sum_{i = 1}^n y_i \psi_i, 
\]
where the $\psi_i$ are defined as in Lemma~\ref{lem:bump-function} with $\vx_0 = \vx_i$ and $\delta = \delta_i/2$. 

By combining the norm bound from Lemma~\ref{lem:bump-function} (4.) with a concentration argument for the $\delta_i$s, we can bound the norm of the above interpolant, and consequently the norm of minimum-norm functions in terms of $n$. 

\begin{corollary}
\label{corr:min-norm-bound}
    Let $\epsilon \in (0,1)$ and $k \in (d/p, 1.5d/p)$. If $n \geq \poly_{\beta, d}(\epsilon^{-1})$, then with probability at least $1 - \epsilon$, a minimum-norm solution satisfies
    \[
    \| f^*\|_{W^{k,p}(\R^d)}^p \lesssim_{k, d, p} n^{kp/d}.
    \]
\end{corollary}

\subsection{Existence of noisy, separated subset}
\label{sec:noisy-separated-subset}

To show high generalization error, we first must identify a subset of data points which are both \textit{noisy}, i.e., have sufficiently high conditional loss compared to the Bayes optimum, and are \textit{separated}, i.e., whose nearest-neighbor radii are sufficiently large. We will also need to show that the associated labels for these data points are not too large, so that individual data points do not have a disproportionate influence on the norm of interpolators.
This result is summarized in the following lemma.

\begin{lemma}\label{lem:delta-packing-size}
    Let $\epsilon  \in (0, 1)$ and let $n \gtrsim \rho^{-2} \log \frac{1}{\epsilon}$. Then with probability at least $1 - \epsilon$, there exists a subset $\mc{B} \subset [n]$ satisfying the following properties:
    \begin{enumerate}
        \item $|\mc{B}| \gtrsim \rho n$.
        \item For all $i \in \mc{B}$, $\delta_i \gtrsim_d n^{-1/d}$.
        \item For all $i \in \mc{B}$,
        \[|y_i| \lesssim \sqrt{1 + \log(\rho^{-1})}.\]
        \item For all $i \in \mc{B}$,
        \[\mc{L}(y_i; \vx_i) \geq \sigma + 
        \mc{L}(f_{\bayes}(\vx_i); \vx_i). \]
    \end{enumerate}
\end{lemma}

With these properties, we can show that the norm bound on the interpolator forces it to accumulate generalization error around these points. 

\subsection{Smoothness of the $\gamma$-ANM solution}
\label{sec:smoothness}

The following lemma characterizes the local deviation of a function in terms of its Sobolev norm.
\begin{corollary}\label{corr:morrey-taylor-specific}
    Let $\delta > 0$ and $d < kp$, and suppose that $u \in W^{k, p}(\R^d)$. Then for all $\vx_0 \in \R^d$ and all $\vx_1 \in B(\vx_0, \delta)$,
    \[|u(\vx_1) - u(\vx_0)|^p \lesssim_{k, d, p} \delta^{kp - d} \|u\|^p_{W^{k,p}(B(\vx_0, 2\delta))}.\]
\end{corollary}
This lemma allows us to show that any interpolant must remain bounded away from the Bayes optimal solution in some neighborhood around the noisy training points. To convert this into information about the loss function, we show that the conditional loss is continuous under our previous assumptions.

\begin{lemma}\label{lem:conditional-loss-continuous}
    The conditional loss $\mc{L}: \R \times \overline{\Omega} \to [0, \infty)$ is continuous.
\end{lemma}

To conclude, we note here that the size of the radii in Lemma~\ref{lem:delta-packing-size} (2.) is critical, as it exactly offsets the scaling of the volumes of the nearest-neighbor balls $B(\vx_i, \delta_i)$, which scale proportional to $\delta_i^d$. By aggregating over a number of data points proportional to $n$, we recover a bound on the generalization error independent of $n$.

\section{Proof of the main theorem}
\label{sec:main-result-proof}

\begin{proof}[Proof of Theorem~\ref{thm:tempered-overfitting-bayes-v2}]
    By Corollary \ref{corr:min-norm-bound}, if $n \geq \poly_{k, p, d}(\epsilon^{-1})$, then with probability at least $1- \frac{\epsilon}{2}$,
    \begin{align}\label{eqn:fstar-norm-bound}
        \|f^*\|^p_{W^{k, p}(\R^d)} \lesssim_{k, d, p} n^{kp/d}.
        % \lesssim_{k, d, p} \sum_{i = 1}^n \left(1 + |y_i|^p \delta_i^{d - k p}\right) 
    \end{align}
    We denote the event that this occurs by $\omega_1$.
    
    By Lemma~\ref{lem:delta-packing-size}, if $n \gtrsim \rho^{-2} \log \frac{1}{\epsilon}$, then with probability at least $1 - \frac{\epsilon}{2}$, there exists a subset $\mc{B} \subset [n]$ satisfying the following properties:
    \begin{enumerate}
        \item $|\mc{B}| \geq C_1\rho n$.
        \item For all $i \in \mc{B}$, $\delta_i \geq C_2 n^{-1/d}$.
        \item For all $i \in \mc{B}$,
        \[|y_i| \leq C_3 \sqrt{1 + \log(\rho^{-1}) }.\]
        \item For all $i \in \mc{B}$,
        \[\mc{L}(y_i; \vx_i) \geq \sigma + 
        \mc{L}(f_{\bayes}(\vx_i); \vx_i). \]
    \end{enumerate}
    Here $C_1, C_3 \in [1, \infty)$ are universal constants, and $C_2 \in (0, \infty)$ is a constant depending only on $d$. We denote the event that this occurs by $\omega_2$. 

    Suppose that both $\omega_1$ and $\omega_2$ occur. Applying Corollary~\ref{corr:morrey-taylor-specific}, we see that for all $i \in \mc{B}$ and $\vx \in B(\vx_i, \delta_i/4)$,
    \begin{align}\label{eqn:f-gamma-local-deviation}
        |f_{\gamma}(\vx) - f_{\gamma}(\vx_i)|^p &\leq C_4 \|\vx - \vx_i\|^{kp - d} \|f_\gamma\|^p_{W^{k,p}(B(\vx_i, \delta_i/2))},
    \end{align}
    where $C_4 > 0$ is a constant depending on $k, d, p$.

    Consider the regret function $\mc{R}: \R \times \overline{\Omega} \to [0, \infty)$, defined by
    \[\mc{R}(y; \vx) = \mc{L}(y; \vx) - \mc{L}(f_{\bayes}(\vx); \vx). \]
    We show that $f_{\gamma}$ attains high regret around many of the training data points. By Theorem~\ref{thm:sobolev-embedding}, $f_{\bayes}$ is continuous, and by Lemma~\ref{lem:conditional-loss-continuous}, $\mc{L}$ is continuous, so $\mc{R}$ is continuous. Let $K = 1 + C_3\sqrt{1 + \log(\rho^{-1})}$. Since $\mc{R}$ is uniformly continuous on $[-K, K] \times \overline{\Omega}$, there exists $\delta \in (0, 1)$ depending on $\mu, \ell$ such that for $(y, \vx), (y', \vx') \in \overline{\Omega}$ with $|y - y'|, \|\vx - \vx'\| \leq \delta$, we have
    \begin{align}\label{eqn:regret-continuity}
        |\mc{R}(y; \vx) - \mc{R}(y'; \vx')| \leq \frac{\sigma}{2}.
    \end{align}
    Let us define
    \begin{align*}
        C_5 = \min\left(\frac{C_2}{2}, \left(\frac{C_1 \delta^p \rho }
        {2C_4}\right)^{1/(kp - d) } \right),
    \end{align*}
    so $C_5$ is a constant depending on $k, p, d, \mu$, and $\ell$. Let $\mc{B}' \subset \mc{B}$ consist of all of the indices $i \in \mc{B}$ such that for all $\vx \in B(\vx_i, C_5 \gamma^{-p/(kp-d)} n^{-1/d})$,
    \begin{align*}
        |f_{\gamma}(\vx) - f_{\gamma}(\vx_i)| &\leq \delta.
    \end{align*}
    Suppose that $i \in \mc{B} \setminus \mc{B}'$. By the definition of $C_5$ and $C_2$,
    \begin{align*}
        C_5 \gamma^{-p/(kp-d)} n^{-1/d} \leq C_5 n^{-1/d} \leq \frac{C_2 n^{-1/d}}{2} \leq \frac{\delta_i}{2},
    \end{align*}
    so $B(\vx_i, C_5 \gamma^{-p/(kp-d)} n^{-1/d}) \subset B(\vx_i, \delta_i/2)$, and therefore for all $\vx \in B(\vx_i, C_5 \gamma^{-p/(kp-d)} n^{-1/d})$,
    \begin{align*}
        \delta^p &\leq |f_{\gamma}(\vx) - f_{\gamma}(\vx_i)|^p\\
        &\leq C_4\|\vx - \vx_i\|^{kp-d}\|f_\gamma\|^p_{W^{k,p}(B(\vx_i, \delta_i/2))}\quad (\text{By (\ref{eqn:f-gamma-local-deviation})})\\
        &\leq C_4 \left(C_5 \gamma^{-p/(kp-d)}n^{-1/d} \right)^{kp - d} \|f_\gamma\|^p_{W^{k,p}(B(\vx_i, \delta_i/2))}\\
        &\leq C_4 \left(\left(\frac{C_1 \delta^p \rho}{2C_4 \gamma^p }\right)^{1 / (kp - d)} n^{-1/d} \right)^{kp - d} \|f_\gamma\|^p_{W^{k,p}(B(\vx_i, \delta_i/2))}\\
        &= \frac{1}{2} C_1 \delta^p \gamma^{-p} \rho n^{1 - kp/d} \|f_\gamma\|^p_{W^{k,p}(B(\vx_i, \delta_i/2))}.
    \end{align*}
    Summing the above over $\mc{B} \setminus \mc{B}'$, we get
    \begin{align*}
        \delta^p|\mc{B} \setminus \mc{B}'| &\leq \frac{1}{2} C_1 \delta^p \gamma^{-p} \rho n^{1 - kp/d}\\&\sum_{ i \in \mc{B} \setminus \mc{B}' } \|f_\gamma\|^p_{W^{k,p}(B(\vx_i, \delta_i/2))}.\\
    \end{align*}
    By Lemma~\ref{lem:delta-packing}, the sets $B(\vx_i, \delta_i/2)$ are disjoint, so
    \begin{align*}
    &\sum_{ i \in \mc{B} \setminus \mc{B}' } \|f_\gamma\|^p_{W^{k,p}(B(\vx_i, \delta_i/2))}\\
        &=\sum_{i \in \mc{B} \setminus \mc{B}' } \sum_{j=1}^k \int_{B(\vx_i, \delta_i/2)}\|D^j f_{\gamma}(\vx)\|^p d\vx \\
        &\leq \sum_{j=1}^k\int_{\cup_{i \in \mc{B} \setminus \mc{B}'} B(\vx_i, \delta_i/2) }\|D^j f_{\gamma}(\vx)\|^p d\vx\\
        &\leq \sum_{j=1}^k \int_{\R^d}\|D^j f_{\gamma}(\vx)\|^p d\vx\\
        &\leq \|f\|_{W^{k,p}(\R^d)}^p
    \end{align*}
    and we can write
    \begin{align*}
        %\delta^p |\mc{B} \setminus \mc{B}'| &\leq \frac{1}{2}C_1 \delta^p \gamma^{-p} \rho n^{1 - kp/d}\int_{\R^d}\|D^k f_\gamma(\vx)\|^p d\vx\\
        \delta^p |\mc{B} \setminus \mc{B}'|&\leq \frac{1}{2}C_1 \delta^p \gamma^{-p} \rho n^{1 - kp/d}\|f_{\gamma}\|_{W^{k, p}(\R^d) }^p \\
        &\leq \frac{1}{2}C_1 \delta^p  \rho n^{1 - kp/d} \|f^*\|^p_{W^{k, p}(\R^d)}\\
        &\leq \frac{1}{2}C_1 \delta^p \rho n \qquad \text{(By (\ref{eqn:fstar-norm-bound}))}.
    \end{align*}
    Rearranging, we get
    \begin{align*}
        |\mc{B} \setminus \mc{B}'| &\leq \frac{1}{2} C_1 \rho n,
    \end{align*}
    so by property 1 of $\mc{B}$,
    \begin{align}\label{eqn:size-of-b-prime}
        |\mc{B}'| = |\mc{B}| - |\mc{B} \setminus \mc{B}'|
        \geq C_1 \rho n - \frac{1}{2}C_1 \rho n
        = \frac{1}{2}C_1 \rho n.
    \end{align}
    Now suppose that $i \in \mc{B}'$ and $\vx \in B(\vx_i, C_5 \gamma^{-p/(kp-d)}n^{-1/d})$. By Property 3 of $\mc{B}$, $|f_{\gamma}(\vx_i)| \leq K - 1$. By the construction of $\mc{B}'$, $|f_{\gamma}(\vx) - f_{\gamma}(\vx_i)| \leq \delta$, and in particular
    \[|f_{\gamma}(\vx)| \leq |f_{\gamma}(\vx_i)| + \delta \leq K. \]
    So $(\vx, f_{\gamma}(\vx)), (\vx_i, f_{\gamma}(\vx_i)) \in [-K, K] \times \overline{\Omega}$, and so
    \begin{align}\label{eqn:regret-over-ball}
    \begin{aligned}
        \mc{R}(f_{\gamma}(\vx); \vx) &\geq \mc{R}(f_{\gamma}(\vx_i); \vx_i)\\
        &-|\mc{R}(f_{\gamma}(\vx); \vx) - \mc{R}(f_{\gamma}(\vx_i); \vx_i)|\\
        &\geq \mc{R}(f_{\gamma}(\vx_i); \vx_i) - \frac{\sigma}{2} \qquad \text{(By (\ref{eqn:regret-continuity}))}\\
        &= \sigma - \frac{\sigma}{2} \qquad \text{(Property 4 of $\mc{B}$)}\\
        &= \frac{\sigma}{2}.
    \end{aligned}
    \end{align}
    We have shown that for all $i \in \mc{B}'$, there is a large enough neighborhood around $\vx_i$ in which $f_{\gamma}$ attains high regret. Aggregating over these regions, we get
    \begin{align*}
        &\int_{\R^d} \mc{R}(f_{\gamma}(\vx); \vx) d\mu_{\vx}(\vx) \\
        &= \int_{\R^d} \mc{R}(f_{\gamma}(\vx); \vx) p_{\vx}(\vx) d\vx\\
        &\geq c_{\mathcal{D}}\int_{\R^d} \mc{R}(f_{\gamma}(\vx); \vx) d\vx \qquad \text{(Assumption~\ref{assump:data-regularity})}\\
        &\geq c_{\mathcal{D}}\sum_{i \in \mc{B}' }\int_{\Omega \cap B(\vx_i, C_5 \gamma^{-p/(kp - d)} n^{-1/d}) } \mc{R}(f_{\gamma}(\vx); \vx) d\vx\\
        &\geq c_{\mathcal{D}}\sum_{i \in \mc{B}' }\int_{\Omega \cap B(\vx_i, C_5 \gamma^{-p/(kp - d)} n^{-1/d}) } \frac{\sigma}{2} d\vx \qquad \text{(By (\ref{eqn:regret-over-ball}))}\\
        &= \frac{1}{2} c_{\mathcal{D}} \sigma \sum_{i \in \mc{B}' }|\Omega \cap B(\vx_i, C_5 \gamma^{-p/(kp - d)} n^{-1/d})|\\
        &\gtrsim_{k, p, d, \mu, \ell}  \sum_{i \in \mc{B}'} \gamma^{-pd/(kp - d) } n^{-1} \qquad \text{(By Theorem~\ref{thm:koskela-intersection})}\\
        &= \gamma^{-pd/(kp - d)}n^{-1}|\mc{B}'|\\
        &\gtrsim_{\mu} \gamma^{-pd/(kp - d)} \qquad \text{(By (\ref{eqn:size-of-b-prime}))}.
    \end{align*}
    So there exists a constant $C_6$ depending on $k, p, d, \mu, \ell$ such that
    \begin{align*}
        \int_{\R^d} \mc{R}(f_{\gamma}(\vx); \vx) d\mu_{\vx}(\vx) &\geq C_6 \gamma^{-pd/(kp - d)}.
    \end{align*}
    To conclude, we rewrite
    \begin{align*}
        &\E[\ell(f_{\gamma}(\vx), y)] - \E[\ell(f_{\bayes}(\vx); y)]\\
        &= \int_{\R^d}\left(\mc{L}(f_{\gamma}(\vx); \vx) - \mc{L}(f_{\bayes}(\vx); \vx) \right)d \mu_{\vx}(\vx)\\
        &= \int_{\R^d} \mc{R}(f_{\gamma}(\vx); \vx) d\mu_{\vx}(\vx)\\
        &\geq C_6 \gamma^{-pd/(kp - d)},
    \end{align*}
    where the first line follows from total expectation. 
\end{proof}

\section{Conclusion} 

In this paper, we study approximately norm-minimizing interpolators of noisy datasets in Sobolev spaces. We show that under certain mild assumptions, any such functions must have positive constant generalization error, even as the number of training samples approaches infinity, that is, any such functions must harmfully overfit to the training data in fixed dimension (in contrast to the very high-dimensional case $d \gg n$ in which one sees benign overfitting). We also demonstrate this result in a common setting with squared loss and Gaussian noise. Our results imply that norm minimization and smoothness are not sufficient for interpolation to generalize well, and suggest that improved generalization could require undertraining models when the dataset has high sample size.

There are a number of possible future directions which expand upon our results. While we study norm minimization in Sobolev spaces, our arguments depend on local control of the oscillation of functions and could potentially be extended to other function spaces which satisfy similar inequalities. We consider regression-type loss functions $\ell$ taking value $0$ if and only if we predict the output label, and future work could investigate classification-type loss functions. Finally, the intermediate regime between overfitting in fixed dimension and overfitting in very high dimensions remains an open area of investigation.

\section*{Impact Statement}

This paper presents work whose goal is to advance the field of machine learning. There are many potential societal consequences of our work, none of which we feel must be specifically highlighted here.

\section*{Acknowledgements}
DN was partially supported by NSF DMS 2408912. AS was partially supported by NSF DMS 2136090. GM was partially supported by DARPA AIQ project HR00112520014, NSF DMS-2145630, NSF CCF-2212520, DFG SPP 2298 project 464109215, and BMFTR in DAAD project 57616814 (SECAI). KK was partially supported by NSF CCF-2212520.

\bibliography{references}

@book{adams2003sobolev,
  title={Sobolev Spaces},
  author={Adams, R.A. and Fournier, J.J.F.},
  isbn={9780080541297},
  series={Pure and Applied Mathematics},
  year={2003},
  publisher={Academic Press}
}

@inproceedings{george2023training,
 author = {George, Erin and Murray, Michael and Swartworth, William and Needell, Deanna},
 booktitle = {Advances in Neural Information Processing Systems},
 noeditor = {A. Oh and T. Neumann and A. Globerson and K. Saenko and M. Hardt and S. Levine},
 pages = {35139--35189},
 publisher = {Curran Associates, Inc.},
 title = {Training shallow {ReLU} networks on noisy data using hinge loss: when do we overfit and is it benign?},
 
 volume = {36},
 year = {2023}
}

@inproceedings{kornowski2023tempered,
title={From Tempered to Benign Overfitting in {ReLU} Neural Networks},
author={Guy Kornowski and Gilad Yehudai and Ohad Shamir},
booktitle={Thirty-seventh Conference on Neural Information Processing Systems},
year={2023},
}

@book{vershynin_2018, 
place={Cambridge}, 
series={Cambridge Series in Statistical and Probabilistic Mathematics}, 
title={High-Dimensional Probability: An Introduction with Applications in Data Science}, 
noDOI={10.1017/9781108231596}, 
 
publisher={Cambridge University Press}, 
author={Vershynin, Roman}, 
year={2018}, 
collection={Cambridge Series in Statistical and Probabilistic Mathematics}
}

@article{zhang2021understanding,
  title={Understanding deep learning (still) requires rethinking generalization},
  author={Zhang, Chiyuan and Bengio, Samy and Hardt, Moritz and Recht, Benjamin and Vinyals, Oriol},
  journal={Communications of the ACM},
  volume={64},
  number={3},
  pages={107--115},
  year={2021},
  publisher={ACM New York, NY, USA}
}

@InProceedings{pmlr-v178-shamir22a,
  title = 	 {The Implicit Bias of Benign Overfitting},
  author =       {Shamir, Ohad},
  booktitle = 	 {Proceedings of Thirty Fifth Conference on Learning Theory},
  pages = 	 {448--478},
  year = 	 {2022},
  noeditor = 	 {Loh, Po-Ling and Raginsky, Maxim},
  volume = 	 {178},
  series = 	 {Proceedings of Machine Learning Research},
  nomonth = 	 {02--05 Jul},
  publisher =    {PMLR}
}

@article{bartlett2020benign,
  title={Benign overfitting in linear regression},
  author={Bartlett, Peter L and Long, Philip M and Lugosi, G{\'a}bor and Tsigler, Alexander},
  journal={Proceedings of the National Academy of Sciences},
  volume={117},
  number={48},
  pages={30063--30070},
  year={2020},
  publisher={National Acad Sciences}
}

@inproceedings{cao2022benign,
 author = {Cao, Yuan and Chen, Zixiang and Belkin, Misha and Gu, Quanquan},
 booktitle = {Advances in Neural Information Processing Systems},
 noeditor = {S. Koyejo and S. Mohamed and A. Agarwal and D. Belgrave and K. Cho and A. Oh},
 pages = {25237--25250},
 publisher = {Curran Associates, Inc.},
 title = {Benign Overfitting in Two-layer Convolutional Neural Networks},
 
 volume = {35},
 year = {2022}
}

@InProceedings{kou2023benign,  
title =  {Benign Overfitting in Two-layer {R}e{LU} Convolutional Neural Networks},  
author =       {Kou, Yiwen and Chen, Zixiang and Chen, Yuanzhou and Gu, Quanquan},  booktitle =  {Proceedings of the 40th International Conference on Machine Learning},  
pages =  {17615--17659},  
year =  {2023},  
noeditor =  {Krause, Andreas and Brunskill, Emma and Cho, Kyunghyun and Engelhardt, Barbara and Sabato, Sivan and Scarlett, Jonathan},  
volume =  {202},  
series =  {Proceedings of Machine Learning Research},  
nomonth =  {23--29 Jul},  
publisher =    {PMLR}}

@ARTICLE{9051968,
  author={Muthukumar, Vidya and Vodrahalli, Kailas and Subramanian, Vignesh and Sahai, Anant},
  journal={IEEE Journal on Selected Areas in Information Theory}, 
  title={Harmless Interpolation of Noisy Data in Regression}, 
  year={2020},
  volume={1},
  number={1},
  pages={67-83},
  nodoi={10.1109/JSAIT.2020.2984716}}

@article{10.1214/21-AOS2133,
author = {Trevor Hastie and Andrea Montanari and Saharon Rosset and Ryan J. Tibshirani},
title = {Surprises in high-dimensional ridgeless least squares interpolation},
volume = {50},
journal = {The Annals of Statistics},
number = {2},
publisher = {Institute of Mathematical Statistics},
pages = {949--986},
year = {2022},
nodoi = {10.1214/21-AOS2133}
}

@InProceedings{pmlr-v134-zou21a,
  title = 	 {Benign Overfitting of Constant-Stepsize {SGD} for Linear Regression},
  author =       {Zou, Difan and Wu, Jingfeng and Braverman, Vladimir and Gu, Quanquan and Kakade, Sham},
  booktitle = 	 {Proceedings of Thirty Fourth Conference on Learning Theory},
  pages = 	 {4633--4635},
  year = 	 {2021},
  noeditor = 	 {Belkin, Mikhail and Kpotufe, Samory},
  volume = 	 {134},
  series = 	 {Proceedings of Machine Learning Research},
  nomonth = 	 {15--19 Aug},
  publisher =    {PMLR}
}

@inproceedings{
koehler2021uniform,
title={Uniform Convergence of Interpolators: {G}aussian Width, Norm Bounds and Benign Overfitting},
author={Frederic Koehler and Lijia Zhou and Danica J. Sutherland and Nathan Srebro},
booktitle={Advances in Neural Information Processing Systems},
noeditor={A. Beygelzimer and Y. Dauphin and P. Liang and J. Wortman Vaughan},
year={2021},

}

@article{JMLR:v23:21-1199,
  author  = {Niladri S. Chatterji and Philip M. Long},
  title   = {Foolish Crowds Support Benign Overfitting},
  journal = {Journal of Machine Learning Research},
  year    = {2022},
  volume  = {23},
  number  = {125},
  pages   = {1--12},
}

@inproceedings{
mallinar2022benign,
title={Benign, Tempered, or Catastrophic: Toward a Refined Taxonomy of Overfitting},
author={Neil Rohit Mallinar and James B Simon and Amirhesam Abedsoltan and Parthe Pandit and Misha Belkin and Preetum Nakkiran},
booktitle={Advances in Neural Information Processing Systems},
noeditor={Alice H. Oh and Alekh Agarwal and Danielle Belgrave and Kyunghyun Cho},
year={2022},

}

@inproceedings{FreiVBS23,
  title = {Benign Overfitting in Linear Classifiers and Leaky {ReLU} Networks from {KKT} Conditions for Margin Maximization},
  author = {Spencer Frei and Gal Vardi and Peter L. Bartlett and Nathan Srebro},
  year = {2023},
  
  researchr = {https://researchr.org/publication/FreiVBS23},
  cites = {0},
  citedby = {0},
  pages = {3173-3228},
  booktitle = {The Thirty Sixth Annual Conference on Learning Theory, 12-15 July 2023, Bangalore, India},
  noeditor = {Gergely Neu and Lorenzo Rosasco},
  volume = {195},
  series = {Proceedings of Machine Learning Research},
  publisher = {PMLR},
}

@inproceedings{
xu2024benign,
title={Benign Overfitting and Grokking in {R}e{LU} Networks for {XOR} Cluster Data},
author={Zhiwei Xu and Yutong Wang and Spencer Frei and Gal Vardi and Wei Hu},
booktitle={The Twelfth International Conference on Learning Representations},
year={2024},

}

@article{eppstein1997nearest,
  title={On nearest-neighbor graphs},
  author={Eppstein, David and Paterson, Michael S and Yao, F Frances},
  journal={Discrete \& Computational Geometry},
  volume={17},
  pages={263--282},
  year={1997},
  publisher={Springer}
}

@inproceedings{buchholz2022kernel,
  title={Kernel interpolation in Sobolev spaces is not consistent in low dimensions},
  author={Buchholz, Simon},
  booktitle={Conference on Learning Theory},
  pages={3410--3440},
  year={2022},
  organization={PMLR}
}

@book{evans2022partial,
  title={Partial differential equations},
  author={Evans, Lawrence C},
  volume={19},
  year={2022},
  publisher={American Mathematical Society}
}

@article{koskela1990capacity,
  author = {P. Koskela},
  title = {Capacity extension domains},
  journal = {Ann. Acad. Sci. Fenn. Math. Diss.},
  year = {1990},
  volume = {73},
  pages = {42},
  note = {42 pp.}
}

@article{haas2023mind,
  title={Mind the spikes: Benign overfitting of kernels and neural networks in fixed dimension},
  author={Haas, Moritz and Holzm{\"u}ller, David and Luxburg, Ulrike and Steinwart, Ingo},
  journal={Advances in Neural Information Processing Systems},
  volume={36},
  pages={20763--20826},
  year={2023}
}

@inproceedings{combes2024extension,
  title={An extension of {M}c{D}iarmid's inequality},
  author={Combes, Richard},
  booktitle={2024 IEEE International Symposium on Information Theory (ISIT)},
  pages={79--84},
  year={2024},
  organization={IEEE}
}

@article{harel2024provable,
  title={Provable tempered overfitting of minimal nets and typical nets},
  author={Harel, Itamar and Hoza, William and Vardi, Gal and Evron, Itay and Srebro, Nati and Soudry, Daniel},
  journal={Advances in Neural Information Processing Systems},
  volume={37},
  pages={53458--53524},
  year={2024}
}

@inproceedings{
barzilai2025beyond,
title={Beyond Benign Overfitting in {N}adaraya-{W}atson Interpolators},
author={Daniel Barzilai and Guy Kornowski and Ohad Shamir},
booktitle={The Thirty-ninth Annual Conference on Neural Information Processing Systems},
year={2025},

}

@article{joshi2023noisy,
  title={Noisy interpolation learning with shallow univariate {ReLU} networks},
  author={Joshi, Nirmit and Vardi, Gal and Srebro, Nathan},
  journal={arXiv preprint arXiv:2307.15396},
  year={2023}
}

@article{cover1967nearest,
  title={Nearest neighbor pattern classification},
  author={Cover, Thomas and Hart, Peter},
  journal={IEEE transactions on information theory},
  volume={13},
  number={1},
  pages={21--27},
  year={1967},
  publisher={IEEE}
}

@article{karhadkar2024benign,
  title={Benign overfitting in leaky {ReLU} networks with moderate input dimension},
  author={Karhadkar, Kedar and George, Erin and Murray, Michael and Montufar, Guido F and Needell, Deanna},
  journal={Advances in Neural Information Processing Systems},
  volume={37},
  pages={36634--36682},
  year={2024}
}

@article{magen2024benign,
  title={Benign overfitting in single-head attention},
  author={Magen, Roey and Shang, Shuning and Xu, Zhiwei and Frei, Spencer and Hu, Wei and Vardi, Gal},
  journal={arXiv preprint arXiv:2410.07746},
  year={2024}
}

@inproceedings{kur2024minimum,
  title={Minimum norm interpolation meets the local theory of banach spaces},
  author={Kur, Gil and Abdalla, Pedro and Bizeul, Pierre and Yang, Fanny},
  booktitle={Forty-first International Conference on Machine Learning},
  year={2024}
}

@inproceedings{rakhlin2019consistency,
  title={Consistency of interpolation with {L}aplace kernels is a high-dimensional phenomenon},
  author={Rakhlin, Alexander and Zhai, Xiyu},
  booktitle={Conference on Learning Theory},
  pages={2595--2623},
  year={2019},
  organization={PMLR}
}

@article{beaglehole2023inconsistency,
  title={On the inconsistency of kernel ridgeless regression in fixed dimensions},
  author={Beaglehole, Daniel and Belkin, Mikhail and Pandit, Parthe},
  journal={SIAM Journal on Mathematics of Data Science},
  volume={5},
  number={4},
  pages={854--872},
  year={2023},
  publisher={SIAM}
}

@article{wang2022binary,
  title={Binary classification of {G}aussian mixtures: Abundance of support vectors, benign overfitting, and regularization},
  author={Wang, Ke and Thrampoulidis, Christos},
  journal={SIAM Journal on Mathematics of Data Science},
  volume={4},
  number={1},
  pages={260--284},
  year={2022},
  publisher={SIAM}
}

@InProceedings{pmlr-v235-cheng24g,
  title = 	 {Characterizing Overfitting in Kernel Ridgeless Regression Through the Eigenspectrum},
  author =       {Cheng, Tin Sum and Lucchi, Aurelien and Kratsios, Anastasis and Belius, David},
  booktitle = 	 {Proceedings of the 41st International Conference on Machine Learning},
  pages = 	 {8141--8162},
  year = 	 {2024},
  noeditor = 	 {Salakhutdinov, Ruslan and Kolter, Zico and Heller, Katherine and Weller, Adrian and Oliver, Nuria and Scarlett, Jonathan and Berkenkamp, Felix},
  volume = 	 {235},
  series = 	 {Proceedings of Machine Learning Research},
  month = 	 {21--27 Jul},
  publisher =    {PMLR}
}

@article{li2024kernel,
  title={Kernel interpolation generalizes poorly},
  author={Li, Yicheng and Zhang, Haobo and Lin, Qian},
  journal={Biometrika},
  volume={111},
  number={2},
  pages={715--722},
  year={2024},
  publisher={Oxford University Press}
}
\bibliographystyle{icml2026}

\appendix
\onecolumn

\section{Background on Sobolev spaces}\label{app:sobolev}

In this section, we briefly review some of the tools from the theory of Sobolev spaces that we employ. For a more complete introduction, see \citet{evans2022partial} and \citet{adams2003sobolev}.

First, recall that for an open set $U \subset \R^d$, $W^{k, p}(U)$ consists of functions $u \in L^p(U)$ which admit weak derivatives $D^{\alpha}u \in L^p(U)$ for all multi-indices $\alpha$ with $|\alpha| \leq k$. Here the weak derivative $D^{\alpha} u$ is defined to be a function in $L^p(U)$ such that for all smooth functions $v: U \to \R$ with compact support,
\begin{align*}
    \int_{U} (D^{\alpha} u)v d\vx &= (-1)^{|\alpha|}\int_U u (D^{\alpha }v)d \vx.
\end{align*}
That is, the weak derivative is defined to be the function such that integration by parts holds.

A main tool in the analysis of Sobolev spaces is the Sobolev embedding theorem, which allows us to understand Sobolev spaces as spaces of smooth functions for the right scaling of $k$, $d$, and $p$. 
\begin{theorem}[Sobolev embedding]\label{thm:sobolev-embedding}
    Let $U$ be a bounded open subset of $\R^d$ with $C^1$ boundary. Suppose that $k > \frac{d}{p}$, and let
    \[\gamma = \begin{cases}
        \left\lfloor \frac{d}{p}\right\rfloor + 1 - \frac{d}{p}  & \text{ if $\frac{d}{p}$ is not an integer,}\\
        \text{any positive number $< 1$}, & \text{if $\frac{d}{p}$ is an integer}.
    \end{cases} \]
    Then there exists $u^* \in C^{k - \lfloor d/p \rfloor -1, \gamma}(\overline{U})$ such that $u^* = u$ almost everywhere, and
    \[\|u^*\|_{C^{k - \lfloor d/p \rfloor - 1, \gamma  }(\overline{U}) } \lesssim_{k, p, \gamma, U } \|u\|_{W^{k, p}(U) }.  \]
\end{theorem}
This also lets us define pointwise evaluation for functions $u \in W^{k, p}(U)$ for open sets $U \subset \R^d$ when $k > \frac{d}{p}$. By the Sobolev embedding theorem, there exists a unique continuous function $u^* \in C^0(\R^d)$ which is equal to $u$ almost everywhere. For a point $\vx_0 \in \R^d$, we define $u(\vx_0) = u^*(\vx_0)$. This evaluation defines a linear functional on $W^{k, p}(U)$. If $U$ is bounded with $C^1$ boundary, this functional is continuous, since
\[ |u(\vx_0)| = |u^*(\vx_0)| \leq \|u^*\|_{C^0(\overline{U})} \lesssim_{k, p, U} \|u\|_{W^{k, p}(U) }.     \]
In particular, this allows us to define pointwise evaluation for functions $u \in W^{k, p}(\R^d)$, since any such function is also an element of $W^{k, p}(U)$ for any open ball $U$ containing the point.

\section{Proof of Corollary~\ref{corr:square-loss}}
\label{app:square-loss-proof}
\begin{proof}[Proof of Corollary~\ref{corr:square-loss}]
    We check each of the assumptions of Theorem~\ref{thm:tempered-overfitting-bayes-v2}.

\textit{Assumption~\ref{assump:growth-rate-loss} (Growth rate of loss function):} With $C_\ell, \tau_\ell = 1$, we have
\begin{align*}
    \ell(\hat{y}, y)&= (\hat{y} - y)^2\\
    &\leq (|\hat{y}| +|y|)^2\\
    &\leq \exp(|\hat{y}| +|y|)\\
    &\leq \exp((1+|\hat{y}|)(1+|y|)).
\end{align*}

\textit{Assumption~\ref{assump:regularity-conditional-distribution} (Regularity of conditional distribution):} By construction, the distribution has a regular conditional probability given by $(\vx, A) \mapsto \mu_{\vx}(A)$, where $\mu_{\vx_0}$ is a Gaussian measure with mean 0 and variance $\sigma(\vx)^2$. If $h: \R \to \R$ is continuous and bounded, then
\begin{align*}
    \lim_{m \to \infty} \int_{\R} g d\nu_{\vx_m} &= \lim_{m \to \infty} \int_{\R} \frac{1}{\sqrt{2\pi \sigma(\vx_m)^2 }}g(\vx) \exp\left(-\frac{\vx^2}{\sigma(\vx_m)^2} \right)d \vx.
\end{align*}
The integrand on the right-hand side is less than
\begin{align*}
    \frac{1}{\sqrt{2\pi \sigma_{\min}^2 }} g(\vx)\exp\left(-\frac{\vx^2}{\sigma_{\max}^2 } \right),
\end{align*}
which is integrable, so by the dominated convergence theorem, the integral is equal to
\begin{align*}
    \int_{\R}\frac{1}{\sqrt{2\pi \sigma(\vx_0)^2 }}g(\vx) \exp\left(-\frac{\vx^2}{\sigma^2(\vx_0)} \right)d \vx &= \int_{\R} gd\nu_{\vx_0},
\end{align*}
and the assumption holds.

As the conditional probability has a continuous and bounded density, this follows from the dominated convergence theorem.

\textit{Assumption~\ref{assump:data-regularity} (Regularity of marginal distribution):} We assume this condition to hold.

\textit{Assumption~\ref{assump:subgaussian-outputs} (Output distribution is conditionally sub-Gaussian):} The output distribution is conditionally Gaussian by assumption, with sub-Gaussian parameter given by $\sigma_{\max}$.

\textit{Assumption~\ref{assump:label-noise} (Label noise):} The conditional loss is
\begin{align*}
    \mc{L}(\hat{y}; \vx) &= \int_{\R} \ell(\hat{y}, y) d\nu_{\vx}(y)\\
    &= \int_{\R} (\hat{y} - y)^2 d\nu_{\vx}(y)\\
    &= \E[(\hat{y} - Z)^2],
\end{align*}
where $Z$ is a Gaussian random variable with mean $g(\vx)$ and variance $\sigma(\vx)^2$. Then
\begin{align}\label{eqn:conditional-loss-gaussian}
    \E[(\hat{y} - Z)^2] &= \text{Var}(\hat{y} - Z) + \E[\hat{y} - Z]^2\\
    &= \sigma(\vx)^2 + (\hat{y} - g(\vx))^2.
\end{align}
So
\[\mc{L}(\hat{y}; \vx) = \sigma(\vx)^2 + (\hat{y} - g(\vx))^2. \]
Taking the infimum of both sides in $\hat{y}$, we get
\begin{align*}
    \inf_{\hat{y} \in \R}\mc{L}(\hat{y}; \vx) &= \sigma(\vx)^2
\end{align*}
with equality attained if and only if $\hat{y} = g(\vx)$. In particular, $g$ is Bayes-optimal. Then for $(\vx, y) \sim \mu$,
\begin{align*}
    \P\left(\mc{L}(y; \vx) \geq \sigma_{\min}^2 + \inf_{\hat{y} \in \R} \mc{L}(\hat{y}; \vx) \middle| \vx\right) &= \P\left((y - g(\vx))^2 \geq \sigma_{\min}^2 \middle| \vx \right)\\
    &= \P(\epsilon^2 \geq \sigma^2_{\min} \mid \vx)\\
    &\geq \P(\epsilon^2 \geq \sigma(\vx)^2 \mid \vx)\\
    &\geq 0.1,
\end{align*}
where $\epsilon$ is a Gaussian random variable with mean $0$ and standard deviation $\sigma(\vx)$. So the assumption is satisfied with constants $\sigma^2_{\min}$ and $0.1$.

\textit{Assumption~\ref{assump:bayes-optimal-regularity}}: We have shown that $g$ is Bayes-optimal, and $g \in W^{k, p}(\R^d)$ by definition.

Therefore, the assumptions of Theorem~\ref{thm:tempered-overfitting-bayes-v2} are satisfied.
Thus, we are guaranteed that there exists a constant $C$ depending on $k, d, p$, and $\mu$ such that
$\E[\ell(f_{\gamma}(\vx), y) - \ell(g(\vx), y)] \geq C$.
We will rewrite the left-hand side of this inequality. By the law of total expectation,
\begin{align*}
    \E[\ell(f_{\gamma}(\vx), y) - \ell(g(\vx), y)] &= \E[\E[\mc{L}(f_{\gamma}(\vx); \vx) - \mc{L}(g(\vx); \vx)  \mid \vx] ]\\
    &= \E[\E[(f_{\gamma}(\vx) - g(\vx))^2 \mid \vx ] ] & \text{(By \ref{eqn:conditional-loss-gaussian})}\\
    &= \E[(f_{\gamma}(\vx) - g(\vx))^2]\\
    &= \|f_{\gamma} - g\|_{L^2(\mu)}^2. 
\end{align*}
This is bounded below by a constant independent of $n$, so the result follows.

\end{proof}

\section{Proofs from Section~\ref{sec:proof-overview}}

In this section, we present the proofs of the results from Section~\ref{sec:proof-overview}.

\subsection{Properties of the conditional loss}

As a first step, we show some basic properties of the conditional loss.

\begin{lemma}\label{lem:exp-integral-finite}
    For all $\tau \geq 0$,
    \begin{align*}
        \sup_{\vx \in \overline{\Omega}} \int_{\R} \exp(\tau |y|) d\nu_{\vx}(y) < \infty.
    \end{align*}
\end{lemma}
\begin{proof}[Proof of Lemma~\ref{lem:exp-integral-finite}]
    Suppose that $(\vx, y) \sim \mu$. By Assumption~\ref{assump:subgaussian-outputs},  $y$ is conditionally sub-Gaussian given $\vx$, so there exist $C_1, C_2 \in (0, \infty)$ (depending on $\tau$) such that 
    \begin{align*}
        2 &\geq \E\left[\exp\left(\frac{y^2}{C_1^2} \right) \middle| \vx \right]\\
        &\geq C_2 \E\left[\exp(\tau |y|) \mid \vx \right]
    \end{align*}
    almost surely. Then by the definition of $\nu_{\vx}$, there exists a subset $\Omega' \subset \overline{\Omega}$ of full measure such that
    \begin{align*}
        \int_{\R} \exp(\tau|y|) d\nu_{\vx}(y) &\leq \frac{2}{C_2}
    \end{align*}
    for all $\vx \in \Omega'$. We show that the above inequality in fact holds for all $\vx \in \overline{\Omega}$. Let $\vx \in \overline{\Omega}$. Since $\Omega'$ has full measure in $\Omega$ and $\Omega$ is an open subset of $\R^d$, $\Omega'$ is dense in $\overline{\Omega}$. Let $\{\vx_m\}_{m \in \N}$ be a sequence in $\Omega'$ converging to $\vx$. Let $K > 0$. Then for all $m \in \N$,
    \begin{align*}
        \frac{2}{C_2} &\geq \int_{\R}\min\left(K, \exp(\tau|y|) \right)d \nu_{\vx_m}(y),
    \end{align*}
    and by Assumption~\ref{assump:regularity-conditional-distribution},
    \begin{align*}
        \lim_{m \to \infty} \int_{\R}\min\left(K, \exp(\tau|y|) \right)d \nu_{\vx_m}(y) &= \int_{\R}\min\left(K, \exp(\tau|y|) \right)d \nu_{\vx}(y).
    \end{align*}
    So for all $K > 0$,
    \begin{align*}
        \int_{\R}\min(K, \exp(\tau|y|)) d\nu_{\vx}(y) \leq \frac{2}{C_2}.
    \end{align*}
    By the monotone convergence theorem,
    \begin{align*}
        \int_{\R}\exp(\tau|y|)d\nu_{\vx}(y) &= \int_{\R} \lim_{K \to \infty}\min(K, \exp(\tau|y|)) d\nu_{\vx}(y)\\
        &= \lim_{K \to \infty} \int_{\R}\min(K, \exp(\tau|y|)) d\nu_{\vx}(y)\\
        &\leq \frac{2}{C_2}.
    \end{align*}
    This holds for all $\vx \in \overline{\Omega}$, so the result follows.
\end{proof}
As defined, the conditional loss could in principle be infinite. We show that in fact it only attains finite values.

\begin{lemma}\label{lem:conditional-loss-finite}
    For all $y_0 \in \R$ and $\vx_0 \in \overline{\Omega}$, $\mc{L}(y_0; \vx_0) < \infty$.
\end{lemma}
\begin{proof}[Proof of Lemma~\ref{lem:conditional-loss-finite}]
    Let $y_0 \in \R$ and $\vx_0 \in \overline{\Omega}$. Then 
    \begin{align*}
        \mc{L}(y_0; \vx_0) &= \int_{\R} \ell(y_0, y) d \nu_{\vx_0}(y)\\
        &\leq \int_{\R}C_{\ell} \exp\left(\tau_{\ell}(1 + |y_0|)(1 + |y|) \right)d \nu_{\vx_0}(y) & \text{(Assumption~\ref{assump:growth-rate-loss})}\\
        &< \infty. & \text{(Lemma~\ref{lem:exp-integral-finite})}
    \end{align*}
\end{proof}

Since $\mc{L}(y; \vx) < \infty$ for all $y \in \R$ and $\vx \in \overline{\Omega}$, we will henceforth view $\mc{L}$ as a function $\R \times \overline{\Omega} \to [0, \infty)$. Importantly, we can show that this function is continuous.

Finally, we conclude that the Bayes optimizer minimizes the conditional loss at each point.

\begin{lemma}\label{lem:bayes-optimal-alternate}
    For all $\vx_0 \in \overline{\Omega}$,
    \[\mc{L}(f_{\bayes}(\vx_0); \vx_0) = \inf_{\hat{y} \in \R } \mc{L}(\hat{y}; \vx_0). \]
\end{lemma}
\begin{proof}[Proof of Lemma~\ref{lem:bayes-optimal-alternate}]
    Let $(\vx, y) \sim \mu$. By Assumption~\ref{assump:bayes-optimal-regularity}, there exists a subset $\Omega' \subset \overline{\Omega}$ of full measure such that for all $\vx_0 \in \Omega$,
    \begin{align*}
        \mc{L}(f_{\bayes}(\vx_0); \vx_0) &= \int_\R \ell(f_{\bayes}(\vx), y) d\nu_{\vx_0}(y)\\
        &= \E[\ell(f_{\bayes}(\vx), y) \mid \vx = \vx_0 ]\\
        &= \inf_{\hat{y} \in \R} \E[\ell(\hat{y}, y) \mid \vx = \vx_0]\\
        &= \inf_{\hat{y} \in \R} \int_{\R} \ell(\hat{y}, y) d\nu_{\vx_0}(y)\\
        &= \inf_{\hat{y} \in \R} \mc{L}(\hat{y}; \vx_0).
    \end{align*}
    So the statement holds for all $\vx_0 \in \Omega'$. Now suppose that $\vx_0 \in \overline{\Omega}$. Since $\Omega'$ has full measure in $\overline{\Omega}$ and $\Omega$ is open, $\Omega'$ is dense in $\overline{\Omega}$. Let $\{\vx_m\}_{m \in \N}$ be a sequence in $\Omega'$ converging to $\vx_0$. By Sobolev embedding (Theorem~\ref{thm:sobolev-embedding}), $f_{\bayes}$ is continuous, and by Lemma~\ref{lem:conditional-loss-continuous}, $\mc{L}$ is continuous, so the map $\vx \mapsto \mc{L}(f_{\bayes}(\vx); \vx)$ is continuous. Let $\hat{y} \in \R$. Then for all $m \in \N$,
    \begin{align*}
        \mc{L}(f_{\bayes}(\vx_m); \vx_m) &\leq \mc{L}(\hat{y}; \vx_m).
    \end{align*}
    Taking the limit of both sides as $m \to \infty$ and applying continuity, we get
    \begin{align*}
        \mc{L}(f_{\bayes}(\vx_0); \vx_0) &\leq \mc{L}(\hat{y}; \vx_0).
    \end{align*}
    Since this holds for all $\hat{y} \in \R$, we have
    \begin{align*}
        \mc{L}(f_{\bayes}(\vx_0); \vx_0) &= \inf_{\hat{y} \in \R} \mc{L}(\hat{y}; \vx_0).
    \end{align*}
    So the condition holds for all $\vx_0 \in \overline{\Omega}$.
\end{proof}

With these results in hand, we can prove that the conditional loss is continuous as stated in Lemma~\ref{lem:conditional-loss-continuous}.

\begin{proof}[Proof of Lemma~\ref{lem:conditional-loss-continuous}]
    Let $\{(y_m, \vx_m)\}_{m \in \N}$ be a sequence in $\R \times \overline{\Omega}$ converging to $(y_0, \vx_0) \in \R \times \overline{\Omega}$. Let $\epsilon > 0$. Let $\phi: [0, \infty) \to [0, 1]$ be a continuous function such that $\phi(t) = 1$ for $t \leq 1$ and $\phi(t) = 0$ for $t \geq 2$. Let $K > 0$. Then for all $\vx \in \overline{\Omega}$ and $\hat{y} \in [y_0 - 1, y_0 + 1]$,
    \begin{align}\label{eqn:cutoff-decomp}
    \begin{aligned}
        \mc{L}(\hat{y}; \vx) &= \int_\R \ell(\hat{y}, y) d\nu_{\vx}(y)\\
        &= \int_{\R} \ell(\hat{y}, y) \phi(K|y|) d \nu_{\vx}(y) + \int_{\R} \ell(\hat{y}, y) (1 - \phi(K|y|)) d\nu_{\vx}(y).
    \end{aligned}
    \end{align}
    By Assumption~\ref{assump:growth-rate-loss}, we can bound the second term as
    \begin{align*}
        \int_{\R} \ell(\hat{y}, y) (1 - \phi(K|y|)) d\nu_{\vx}(y) &\leq \int_\R \1_{|y| \geq K}  \ell(\hat{y}, y) d\nu_{\vx}(y)\\
        &\leq \int_\R \1_{|y| \geq K} C_{\ell}  \exp\left(\tau_{\ell}(1 + |\hat{y}|)(1 + |y|) \right) d\nu_{\vx}(y)\\
        &\leq \int_\R \1_{|y| \geq K} C_{\ell}  \exp\left(\tau_{\ell}(2 + |y_0|)(1 + |y|) \right) d\nu_{\vx}(y).
    \end{align*}
    By Lemma~\ref{lem:exp-integral-finite},
    \begin{align*}
        \int_{\R} C_{\ell}  \exp\left(\tau_{\ell}(2 + |y_0|)(1 + |y|) \right) d\nu_{\vx}(y) < \infty,
    \end{align*}
    so there exists $K \in (0, \infty)$ such that for all $\vx \in \overline{\Omega}$,
    \begin{align*}
        \int_\R \1_{|y| \geq K} C_{\ell}  \exp\left(\tau_{\ell}(1 + |y_0|)(1 + |y|) \right) d\nu_{\vx}(y) &\leq \frac{\epsilon}{6}.
    \end{align*}
    We fix this value of $K$ for the remainder of the proof. Substituting into (\ref{eqn:cutoff-decomp}), we get
    \begin{align}\label{eqn:l-cutoff-approx}
        \left|\mc{L}(\hat{y}; \vx) - \int_\R \ell(\hat{y}, y)\phi(K|y|) d\nu_{\vx}(y) \right| &\leq \frac{\epsilon}{6}
    \end{align}
    for all $\vx \in \overline{\Omega}$ and $\hat{y} \in [y_0 - 1, y_0 + 1]$. The function $y \mapsto \ell(y_0, y) \phi(K|y|)$ is continuous and bounded, so by Assumption~\ref{assump:regularity-conditional-distribution}, there exists $M_1 \in \N$ such that for $m \geq M_1$,
    \begin{align}\label{eqn:cutoff-function-weak-convergence}
        \left|\int_{\R}\ell(y_0, y)\phi(K|y|)d\nu_{\vx_m}(y) - \int_{\R}\ell(y_0, y)\phi(K|y|)d\nu_{\vx}(y) \right| \leq \frac{\epsilon}{6}.
    \end{align}
    Combining (\ref{eqn:l-cutoff-approx}) and (\ref{eqn:cutoff-function-weak-convergence}), we obtain
    \begin{align}\label{eqn:L-y0-xm-x0}
        \begin{aligned}
        |\mc{L}(y_0; \vx_m) - \mc{L}(y_0; \vx_0)| &\leq \left|\mc{L}(y_0; \vx_m) - \int_\R \ell(y_0, y)\phi(K|y|) d\nu_{\vx_m}(y) \right|\\
        &+ \left|\int_{\R}\ell(y_0, y)\phi(K|y|)d\nu_{\vx_m}(y) - \int_{\R}\ell(y_0, y)\phi(K|y|)d\nu_{\vx}(y) \right|\\
        &+ \left|\mc{L}(y_0; \vx_0) - \int_\R \ell(y_0, y)\phi(K|y|) d\nu_{\vx_0}(y) \right|\\
        &\leq \frac{\epsilon}{6} + \frac{\epsilon}{6} + \frac{\epsilon}{6}\\
        &= \frac{\epsilon}{2}
        \end{aligned}
    \end{align}
    for $m \geq M_1$.

    There exists $M_2 \in \N$ such that for $m \geq M_2$, $y_m \in [y_0 - 1, y_0 + 1]$. Since $\ell$ is uniformly continuous on the compact set $[y_0 - 1, y_0 + 1] \times [-2K, 2K]$, there exists $M_3 \in \N$ such that for $m \geq M_3$,
    \begin{align}\label{eqn:sup-l-ym-rectangle}
        \sup_{y \in [-2K, 2K]} |\ell(y_m, y) - \ell(y_0, y) | \leq \frac{\epsilon}{24K}.
    \end{align}
    Now suppose that $m \geq \max(M_2, M_3)$. Then by (\ref{eqn:l-cutoff-approx}) and (\ref{eqn:sup-l-ym-rectangle}),
    \begin{align}\label{eqn:L-ym-y0-xm-split}
    \begin{aligned}
        |\mc{L}(y_m; \vx_m) - \mc{L}(y_0; \vx_m)| &\leq \int_{\R}\left|\ell(y_m, y) - \ell(y_0, y)  \right| d \nu_{\vx_m}(y)\\
        &= \int_{\R} (1 -  \phi(K|y|))\left|\ell(y_m, y) - \ell(y_0, y)  \right| d \nu_{\vx_m}(y) \\&+ \int_{\R} \phi(K|y|)\left|\ell(y_m, y) - \ell(y_0, y)  \right| d \nu_{\vx_m}(y)
    \end{aligned}
    \end{align}
    We bound the two terms on the right-hand side separately. For the first term, we have
    \begin{align}\label{eqn:difference-cutoff-remainder}
        \begin{aligned}
            \int_{\R} (1 -  \phi(K|y|))\left|\ell(y_m, y) - \ell(y_0, y)  \right| d \nu_{\vx_m}(y) &\leq \int_{\R} (1 - \phi(K|y|)) \ell(y_m, y) d\nu_{\vx_m}(y)\\
            &+ \int_{\R}(1 - \phi(K|y|)) \ell(y_0, y) d\nu_{\vx_m}(y)\\
            &\leq \frac{\epsilon}{6} + \frac{\epsilon}{6} & \text{By (\ref{eqn:l-cutoff-approx}) }\\
            &= \frac{\epsilon}{3}.
        \end{aligned}
    \end{align}
    For the second term, we have
    \begin{align}\label{eqn:difference-cutoff-main}
    \begin{aligned}
        \int_\R \phi(K|y|)|\ell(y_m, y) - \ell(y_0, y)| d\nu_{\vx_m}(y) &\leq \int_{-2K}^{2K}|\ell(y_m, y) - \ell(y_0, y)| d\nu_{\vx_m}(y)\\
        &\leq \int_{-2K}^{2K}  \frac{\epsilon}{12K} d\nu_{\vx_m}(y) & \text{(By (\ref{eqn:sup-l-ym-rectangle}))} \\
        &\leq \frac{\epsilon}{6}.
    \end{aligned}
    \end{align}
    Substituting (\ref{eqn:difference-cutoff-remainder}) and (\ref{eqn:difference-cutoff-main}) into (\ref{eqn:L-ym-y0-xm-split}), we get
    \begin{align}\label{eqn:L-ym-y0-xm}
        |\mc{L}(y_m; \vx_m) - \mc{L}(y_0; \vx_m)| \leq \frac{\epsilon}{3} + \frac{\epsilon}{6} = \frac{\epsilon}{2}.
    \end{align}
    Combining (\ref{eqn:L-y0-xm-x0}) and (\ref{eqn:L-ym-y0-xm}), we see that for $m \geq \max(M_1, M_2, M_3)$,
    \begin{align*}
        |\mc{L}(y_m; \vx_m) - \mc{L}(y_0; \vx_0)| &\leq |\mc{L}(y_m; \vx_m) - \mc{L}(y_0; \vx_m)| + |\mc{L}(y_0; \vx_m) - \mc{L}(y_0; \vx_0)|\\
        &\leq \frac{\epsilon}{2} + \frac{\epsilon}{2}\\
        &= \epsilon.
    \end{align*}
    Since $\epsilon$ was arbitrary, it follows that $\mc{L}(y_m; \vx_m) \to \mc{L}(y_0; \vx_0)$, and therefore $\mc{L}$ is continuous.
\end{proof}

\subsection{Existence of bump functions}

In this section, we will review that bump functions with small enough norm exist in Sobolev spaces. If $V$ is a Banach space, we say that a map $\psi: \R^d \to V$ is \emph{radially symmetric} if for all $\vx, \vx' \in U$ with $\|\vx\| = \|\vx'\|$, we have $\psi(\vx) = \psi(\vx')$.

\begin{lemma}\label{lem:radially-symmetric-factors}
    If $V$ is a Banach space and $\psi: \R^d \to V$ is a smooth radially symmetric map, then there exists a smooth function $\varphi: \R \to V$ such that for all $\vx \in \R^d$, we have $\psi(\vx) = \varphi\left(\|\vx\|\right).$
\end{lemma}
\begin{proof}[Proof of Lemma~\ref{lem:radially-symmetric-factors}]
    Let $\vx_0 \in \R^d$ be such that $\|\vx_0\| = 1$. Let us define $\varphi: \R \to V$ by $\varphi(t) = \psi(t \vx_0 )$. Then $\varphi$ is smooth. If $\|\vx\| = t$, then by the radial symmetry of $\psi$,
    \begin{align*}
        \psi(\vx) = \psi(t \vx_0) = \varphi(t) = \varphi(\|\vx\|).
    \end{align*}
    This holds for all $\vx \in \R^d$, so $\varphi$ satisfies the desired property.
\end{proof}

\begin{lemma}\label{lem:norm-derivative-radially-symmetric}
    If $V$ is a Banach space and $\psi: \R^d \to V$ is a smooth radially symmetric map, then $\|D^k \psi\|: \R^d \to \mc{B}((\R^d)^{\otimes k}, V) $ is radially symmetric. Here $\mc{B}$ denotes the space of bounded linear maps.
\end{lemma}
\begin{proof}[Proof of Lemma~\ref{lem:norm-derivative-radially-symmetric}]
    Let $\vx_1, \vx_2 \in \R^d \setminus \{0\}$ with $\|\vx_1\| = \|\vx_2\|$. Let $\mR: \R^d \to \R^d$ be an orthogonal linear map with $\mR \vx_1 = \vx_2$. Since $\psi$ is radially symmetric, $\psi \circ \mR = \psi$. By the chain rule,
    \begin{align*}
        (D^k \psi)(\vx)(\bm{\xi}_1, \cdots, \bm{\xi}_k) &= (D^k(\psi \circ R))(\vx)(\bm{\xi}_1, \cdots, \bm{\xi}_k)\\
        &= (D^k \psi)(\mR \vx)(\mR \bm{\xi}_1, \cdots, \mR \bm{\xi}_k),
    \end{align*}
    so $\|(D^k \psi(\vx))\| = \|(D^k \psi)(\mR \vx)\|$ for all $\vx \in \R^d$. Setting $\vx = \vx_1$ yields $\|(D^k \psi)(\vx_1)\| = \|(D^k \psi)(\vx_2)\|$. Hence $\|D^k \psi\|$ is radially symmetric.
\end{proof}

Using the previous two lemmas, we can explicitly construct bump functions with bounded Sobolev norms as shown in Lemma~\ref{lem:bump-function}.

\begin{proof}[Proof of Lemma~\ref{lem:bump-function}] It suffices to prove the statement for $\vx_0 = \bm{0}_d$; the general case follows from composing $\psi$ with a translation. Let $\varphi: \R \to [0, 1]$ be a smooth function supported on $[-1, 1]$ with $\varphi(x) = 1$ for $|x| \leq 1/4$. Let us define $\psi_{\delta}: \R^d \to \R$ by $\psi_{\delta}(\vx) = \varphi\left(\frac{\|\vx\|^2}{\delta^2}\right)$. Then $\psi_{\delta}$ is supported on $B(\vx_0, \delta)$, and $\psi(\vx) = 1$ if $\|\vx - \vx_0\| \leq \frac{\delta}{2}$. The map $\psi_{\delta}$ is a composition of the maps $\vx \mapsto \tfrac{\vx}{\delta}$ and $\psi_1$. So by the chain rule, $(D^j \psi_{\delta})(\vx) = (\delta^{-j} D^j \psi_1)\left(\tfrac{\vx}{\delta}\right)$ for all $j \leq k$ and $\vx \in \R^d$. Observe that $\psi_1$ is radially symmetric, so by Lemma~\ref{lem:norm-derivative-radially-symmetric}, $\|D^j \psi_1\|$ is radially symmetric. By Lemma~\ref{lem:radially-symmetric-factors}, there exists a smooth function $g_{j, d}: \R \to \R$ such that $\|D^j \psi_1(\vx)\| = g_{j, d}(\|\vx\|)$ for all $\vx \in \R^d$. Note that $g_{j, d}$ is supported on $(-\infty, 1]$. Then 
    \begin{align*}
        \|D^j \psi_{\delta}\|_{L^p(\R^d)}^p &= \int_{\R^d } \|D^j \psi_{\delta}(\vx)\|^pd\vx\\
        &=
        \int_{\R^d} \delta^{-jp}\|D^j \psi_1\left(\tfrac{\vx}{\delta} \right)\|^p d\vx \\
        &= \int_{\R^d} \delta^{-jp} g_{j, d}\left(\tfrac{\|\vx\|}{\delta} \right)^p d\vx\\
        &= \int_0^{\infty}\int_{\partial B_r(\bm{0}_d)} \delta^{-jp} g_{j, d}\left(\tfrac{r}{\delta} \right)^p dS(\vx)dr \\
        &\lesssim_d \delta^{-jp} \int_0^{\infty}r^{d - 1} g_{j, d}\left(\tfrac{r}{\delta}\right)^p dr.
    \end{align*}
    By a change of variables, the last line above is equal to
    \[\delta^{-jp + 1}\int_0^{\infty}(\delta r)^{d - 1} g_{j, d}(r)^p dr.  \]
    Since $g_{j, d}$ is supported on $(-\infty, 1]$, the above expression is equal to
    \begin{align*}
        \delta^{-jp + 1} \int_0^1 (\delta r)^{d - 1} g_{j, d}(r)^p dr &= \delta^{d - jp}\int_0^1 r^{d - 1} g_{j, d}(r)^p dr\\
        &\lesssim_{k, d, p} \delta^{d - jp}.
    \end{align*}
    Hence,
    \[\|D^j \psi_{\delta}\|_{L^p(\R^d)} \leq c_{k, d, p} \delta^{(d - jp)/p} \leq c_{k, d, p}\max\left(1, \delta^{(d - jp)/p}\right). \]
    Summing over $j$, we get
    \begin{align*}
        \|\psi_{\delta}\|_{W^{k, p}} &\lesssim_{k, d, p} \sum_{j = 0}^k \|D^j \psi_{\delta}\|_{L^p(\R^d)}\\
        &\lesssim_{k, d, p} \sum_{j = 0}^k \max\left(1, \delta^{(d - jp)/p}\right)\\
        &\lesssim_{k, d, p}\sum_{j = 0}^k \max\left(1, \delta^{(d - kp)/p} \right)\\
        &\lesssim_{k, d, p} \max\left(1, \delta^{(d - kp)/p} \right)\\
        &\lesssim_{k, d, p}1 + \delta^{(d - kp)/p}.
    \end{align*}
    So $\psi_{\delta}$ satisfies the desired properties.
\end{proof}

As mentioned in Section~\ref{sec:interpolating-solution}, we need Lemma~\ref{lem:delta-packing} to ensure that the supports of the bump functions are disjoint.

\begin{proof}[Proof of Lemma~\ref{lem:delta-packing}]
    Let $i, j \in [n]$ with $i \neq j$. Without loss of generality, suppose that $\delta_i \leq \delta_j$. Suppose that there exists $\vx \in B(\vx_i, \delta_i/2) \cap B(\vx_j, \delta_j/2)$. Then by the triangle inequality, $\|\vx_i - \vx_j\| < \frac{\delta_i}{2} + \frac{\delta_j}{2} \leq \delta_j$. This contradicts the definition of $\delta_j$. So $B(\vx_i, \delta_i/2) \cap B(\vx_j, \delta_j/2) = \emptyset$.
\end{proof}

Finally, we can use our interpolant to partially bound the min norm solution.

\begin{lemma}\label{lem:min-norm-delta-bound}
    Let $f^*$ be a solution to (\ref{eqn:min-norm-sobolev}). Then
    \[\| f^*\|_{W^{k, p}}^p \lesssim_{k, d, p} \sum_{i =1}^n\left(1 +  |y_i|^p \delta_i^{d - kp}\right).\]
\end{lemma}
\begin{proof}[Proof of Lemma~\ref{lem:min-norm-delta-bound}] By Lemma~\ref{lem:bump-function}, for each $i \in [n]$, let $\psi_i \in C^{\infty}(\R^d)$ be a function supported on $B(\vx_i, \delta_i/2)$ with $\psi_i(\vx_i) = 1$ and $\| \psi_i\|_{W^{k, p}} \lesssim_{k, d, p} 1 + \delta_i^{(d - kp)/p} $. Then by the definition of $\delta_i$, $\psi_i(\vx_{\ell}) = 0$ for all $\ell \neq i$. 
     
     Let us define $f \in W^{k, p}(\R^d)$ by
    \[f = \sum_{i = 1}^n y_i \psi_i. \]
    Then $f(\vx_i) = y_i$ for all $i \in [n]$. By Lemma~\ref{lem:delta-packing}, the sets $B(\vx_i, \delta_i/2)$ are disjoint, so for all $j \in \{0, \cdots, k\}$,
    \begin{align*}
        \|D^j f\|_{L^p(\R^d)}^p &\lesssim_{k, d, p} \left\|\sum_{i = 1}^n |y_i| D^j \psi_i \right\|_{L^p(\R^d)}^p\\
        &=  \sum_{i = 1}^n |y_i|^p \|D^j \psi_i\|_{L^p(B(\vx_i, \delta_i/2)) }^p\\
        &\lesssim_{k, d, p, \Omega} \sum_{i = 1}^n (1 + |y_i|^p\delta_i^{d - kp}).
    \end{align*}

    Summing over all $j$, we get
    \begin{align*}
        \|f\|_{W^{k, p}}^p &\lesssim_{k, d, p} 1 + \sum_{i = 1}^n(1 + |y_i|^p\delta_i^{d - kp}).
    \end{align*}
    Finally, by the definition of $f^*$,
    \begin{align*}
        \| f^*\|_{W^{k, p}(\R^d)}^p &\leq \|f\|_{W^{k, p}(\R^d)}^p\\
        &\lesssim_{k, d, p} \sum_{i =1}^n\left(1 +  |y_i|^p \delta_i^{d - kp}\right).
    \end{align*}
\end{proof}

\subsection{Properties of the nearest neighbor balls}
The previous section bounds minimum-norm solutions using the nearest neighbor radii, however, the bound in Corollary~\ref{corr:min-norm-bound} relies on concentration inequalities to control the $\delta_i$. We would like to apply a law of large numbers, but the $\delta_i$s are not independent, since they each depend on the entire dataset. We instead show that changing a single data point does not alter too many of the $\delta_i$s. We can then apply McDiarmid's inequality to bound their sum concentration around the mean.

We will need the following version of McDiarmid's inequality for high-probability subsets \citep[Proposition 2]{combes2024extension}.
\begin{theorem}\label{thm:mcdiarmid}
    Let $\mc{X}_1, \cdots, \mc{X}_n$ be measurable spaces, and let $X_1, \cdots, X_n$ be independent random variables with $X_i$ taking values in $\mc{X}_i$. Let $\mc{Y} \subset \mc{X}_1 \times \cdots \times \mc{X}_n$ be measurable, and let us define
    \[p = \P((X_1, \cdots, X_n) \notin \mc{Y}). \]
    
    Let $C_1, \cdots, C_n > 0$. Let $g: \mc{X}_1 \times \cdots \times \mc{X}_n \to \R$ be a measurable function such that for all $i \in [n]$ and all $x_1 \in \mc{X}_1, \cdots, x_n \in \mc{X}_n, x_i' \in \mc{X}_i$ with $(x_1, \cdots, x_n), (x_1, \cdots, x_{i -1 }, x_i', x_{i + 1}, \cdots, x_n) \in \mc{Y}$,
    \[|g(x_1, \cdots, x_{i - 1}, x_i, x_{i + 1}, \cdots, x_n) - g(x_1, \cdots, x_{i - 1}, x_i', x_{i +1}, \cdots, x_n)| \leq C_i. \]
    Then for all $t > 0$,
    \begin{align*}
        &\P\left(g(X_1, \cdots, X_n) - \E[g(X_1, \cdots, X_n) | (X_1, \cdots, X_n) \in \mc{Y}] \geq t + p\sum_{i = 1}^n C_i \right) \leq p + \exp\left(-\frac{2t^2}{\sum_{i=1}^n C_i^2}\right).
    \end{align*}
\end{theorem}

For a dataset $\mX \in \R^{n \times d}$ with distinct points, we define its \emph{nearest neighbor graph} $NN(\mX)$ to be the directed graph with vertex set $\vx_1, \cdots, \vx_n$, with a directed edge from $\vx_i$ to $\vx_j$ if $\|\vx_j - \vx_i\| \leq \|\vx_{\ell} - \vx_i\|$ for all $\ell \neq i$.

\begin{lemma}\label{lem:nn_indegree}
    Let $\mX \in \R^{n \times d}$ be a dataset with distinct points. Then for all $\ell \in [n]$, the in-degree $\deg^-(\vx_{\ell})$ satisfies
    \[\deg^-(\vx_{\ell}) \lesssim_d 1. \]
\end{lemma}
\begin{proof}[Proof of Lemma~\ref{lem:nn_indegree}]
    Let $\mc{N}$ denote the set all $i \in [n]$ such that $(i, \ell)$ is an edge of $NN(\mX)$. We follow the argument by \cite{eppstein1997nearest} relating the cardinality of $\mc{N}$ to the \emph{kissing number} in dimension $d$. We define the kissing number $\tau(d)$ to be the maximum value $M$ such that there exist points $\vu_1, \cdots, \vu_M$ with $\|\vu_i\| = 1$ for all $i$, and $\|\vu_i - \vu_j\| \geq 1$ for all $i \neq j$. The kissing number is finite for all $d \in \mathbb{N}$. To see this, observe that the sets $B(\vu_i, 1)$ are disjoint and contained in $B(\bm{0}_d, 2)$ which has finite measure.

    Now for $i \in \mc{N}$ let us define $\vw_i = \vx_i - \vx_{\ell}$ and  $\vu_i = \frac{\vw_i}{\|\vw_i\|}$. Suppose that $i, j \in \mc{N}$ with $\|\vx_i - \vx_{\ell}\|^2 \leq \|\vx_j  -\vx_{\ell}\|^2$. Since $i, j \in \mc{N}$, we have $\|\vx_j - \vx_{\ell}\|^2 \leq \|\vx_i - \vx_j\|^2$. Then
    \begin{align*}
        2\langle \vx_i - \vx_{\ell}, \vx_j - \vx_{\ell} \rangle &= \|\vx_i - \vx_{\ell}\|^2 + \|\vx_j - \vx_{\ell}\|^2 - \|\vx_i - \vx_j\|^2\\
        &\leq \|\vx_i - \vx_{\ell}\|^2.
        \end{align*}
    In other words,
    \[2 \langle \vw_i, \vw_j \rangle \leq \|\vw_i\|^2. \]
    Dividing both sides by $2\|\vw_i\|\|\vw_j\|$, we get
    \begin{align*}
        \langle \vu_i, \vu_j \rangle \leq \frac{1}{2}\frac{\|\vw_i\| }{\|\vw_j\|} \leq \frac{1}{2}. 
    \end{align*}
    Then
    \begin{align*}
        \|\vu_i - \vu_j\|^2 &= \|\vu_i\|^2 + \|\vu_j\|^2 - 2 \langle \vu_i, \vu_j \rangle\\
        &= 1 + 1 - 2\langle \vu_i, \vu_j \rangle\\
        &\geq 1.
    \end{align*}
    So the set $\mc{N}' = \{\vu_i: i \in \mc{N}\}$ consists of unit vectors with $\|\vu_i - \vu_j\| \geq 1$ for all $i \neq j$, and therefore has cardinality at most $\tau(d)$.

    We claim that for all $i, j \in \mc{N}$ with $i \neq j$, $\vu_i \neq \vu_j$. To see this, suppose otherwise. Then there exist $i, j \in \mc{N}$ with $i \neq j$ and $\vx_i - \vx_{\ell} = \lambda(\vx_j - \vx_{\ell})$ for some $\lambda > 1$. This implies that
    \begin{align*}
        \|\vx_i - \vx_j\| &= \|(\vx_i - \vx_{\ell}) - (\vx_j - \vx_{\ell})\|\\
        &= \left\|\left(1 - \frac{1}{\lambda} \right)(\vx_i - \vx_{\ell})\right\|\\
        &< \|\vx_i - \vx_{\ell}\|.
    \end{align*}
    This contradicts that $\vx_{\ell}$ is the closest point in the dataset to $\vx_i$. So for all $i \neq j$, $\vu_i \neq \vu_j$. This implies that $|\mc{N}'| = |\mc{N}|$. Putting everything together, we have
    \[\deg^-(\vx_{\ell}) = |\mc{N}| = |\mc{N}'| \leq \tau(d). \]
\end{proof}

In order to use McDiarmid's inequality, we must show that the functions of the $\delta_i$s satisfy a bounded difference property. To do so, we show that changing any single data point cannot change too many of the nearest neighbor radii. 

\begin{lemma}\label{lem:not-too-many-deltas-change}
    Let $\mX \in \R^{n \times d}$ be a dataset with distinct points and fix $\ell \in [n]$. Let $\mX' \in \R^{n \times d}$ be a dataset with distinct points such that $\vx_i = \vx'_i$ for all $i \neq \ell$. Then
    \[|\{i \in [n]: \delta_i(\mX) \neq \delta_i(\mX')\}| \lesssim_d 1. \]
\end{lemma}
\begin{proof}[Proof of Lemma~\ref{lem:not-too-many-deltas-change}]
    Suppose that $\delta_i(\mX) \neq \delta_i(\mX')$ for $i \neq \ell$. Then the closest point to $\vx_i$ changes as we modify $\vx_{\ell}$ to $\vx_{\ell}'$. This can only happen if $\vx_{\ell}$ is the nearest neighbor of $\vx_i$ in the dataset $\mX$, or if $\vx_{\ell}'$ is the nearest neighbor of $\vx_i$ in the dataset $\mX'$. In other words, either $(\vx_i, \vx_{\ell})$ is an edge in $NN(\mX)$, or $(\vx_i, \vx_{\ell}')$ is an edge in $NN(\mX')$. By Lemma~\ref{lem:nn_indegree}, the number of such $i \in [n]$ is at most a constant $C_d$. So
    \begin{align*}
        |\{i \in [n]: \delta_i(\mX) \neq \delta_i(\mX')| &\leq 1 + |\{i \in [n] \setminus \{\ell\}: \delta_i(\mX) \neq \delta_i(\mX')\}|\\
        &\leq 1 + 2C_d\\
        &\lesssim_d 1.
    \end{align*}
\end{proof}

We can then apply the previous lemma to bound the sum in Lemma~\ref{lem:min-norm-delta-bound}.

\begin{lemma}\label{lem:scale-of-delta-concentrates}
     Let $\beta \in \left(0, d/2\right)$ and let $\epsilon \in (0, 1)$. If $n \geq \poly_{\beta, d}(\epsilon^{-1})$, then with probability at least $1 - \epsilon$,
    \[\sum_{i = 1}^n  |y_i|^p \delta_i^{-\beta} \lesssim_{d, \beta}  n^{\beta/d + 1}.\]
\end{lemma}
\begin{proof}[Proof of Lemma~\ref{lem:scale-of-delta-concentrates}]
    Before we bound the sum of the terms, we bound the maximum.

    For all $i \in [n]$ and $t > 0$,
    \begin{align}\label{eqn:delta-beta-cdf}
        \begin{aligned}
        \P(\delta_i^{-\beta} \geq t)
        &= \P(\delta_i \leq t^{-1/\beta})\\
        &\leq \sum_{j \neq i} \P(\|\vx_i - \vx_j\| \leq t^{-1/\beta})\\
        &\leq \sum_{j \neq i} C_1  t^{-d/\beta} & \text{(By Assumption~ \ref{assump:data-regularity})}\\
        &\leq C_1 n t^{-d/\beta},
        \end{aligned}
    \end{align}
    where $C_1$ is a constant depending on $d$ and $\mc{D}$.
    By a union bound,
    \begin{align*}
        \P(\max(\delta_1^{-\beta}, \cdots, \delta_n^{-\beta}) \geq t) &\leq C_1 n^2 t^{-d/\beta}.
    \end{align*}
    Setting $t = \left(\frac{4C_1 n^2}{\epsilon}\right)^{\beta/d}$, we get that with probability at least $1 - \frac{\epsilon}{4}$,
    \begin{align}\label{eqn:max-deltai-bound}\max(\delta_1^{-\beta}, \cdots, \delta_n^{-\beta}) < \left(\frac{4C_1 n^2}{\epsilon} \right)^{\beta/d}. \end{align}

    By Assumption~\ref{assump:subgaussian-outputs}, there exists a constant $C_2 > 0$ such that for all $i \in [n]$ and $t \geq 0$,
    \[\P(|y_i| \geq t) \leq 2\exp\left(-\frac{t^2}{C_2^2} \right). \]
    By a union bound,
    \[\P(\max(|y_1|, \cdots, |y_n|) \geq t) \leq 2n\exp\left(-\frac{t^2}{C_2^2} \right). \]
    Setting $t = C_2 \sqrt{\log \frac{8n}{\epsilon}}$ yields that with probability at least $1 - \frac{\epsilon}{4}$,
    \begin{align}\label{eqn:max-yi-bound}\max(|y_1|, \cdots, |y_n|) < C_2 \sqrt{\log \frac{8n}{\epsilon}}. \end{align}
    Now we apply the high probability McDiarmid's inequality (Theorem~\ref{thm:mcdiarmid}). Let $\mc{X} = \Omega \times \R$. Let $g: \mc{X}^n \to \R$ be defined by
    \begin{align*}
        g((\vx_1, y_1), \cdots, (\vx_n, y_n)) = \sum_{i = 1}^n |y_i|^p \delta_i(\vx_1, \cdots, \vx_n)^{-\beta}
    \end{align*}
    if the $\vx_i$ are all distinct. If $\vx_i = \vx_j$ for some $i, j \in [n]$ with $i \neq j$, we arbitrarily define $g((\vx_1, y_1), \cdots, (\vx_n, y_n)) = 0$. We similarly define a thresholded function $\tilde{g}: \mc{X}^n \to \R$ by
    \begin{align*}
        &\tilde{g}((\vx_1, y_1), \cdots, (\vx_n, y_n)) = \sum_{i = 1}^n \min\left(|y_i|, C_2 \sqrt{\log \frac{8n}{\epsilon}} \right)^p\min\left(\delta_i(\vx_1, \cdots, \vx_n)^{-\beta}, \left(\frac{4 C_1 n^2}{\epsilon} \right)^{\beta/d} \right).
    \end{align*}
    By (\ref{eqn:max-deltai-bound}) and (\ref{eqn:max-yi-bound}), with probability at least $1 - \frac{\epsilon}{2}$ we have
    \[g((\vx_1, y_1), \cdots, (\vx_n, y_n)) = \tilde{g}((\vx_1, y_1), \cdots, (\vx_n, y_n)). \]
    We denote the event that this occurs by $\omega_1$.
    
    Let $\mc{Y} \subset \mc{X}^n$ be the set of all $((\vx_1, y_1), \cdots, (\vx_n, y_n)) \in (\Omega \times \R)^n$ such that the points $\vx_1, \cdots, \vx_n$ are distinct, so that
    \begin{align*}
        \P(((\vx_1, y_1), \cdots, (\vx_n, y_n)) \notin \mc{Y}) = 0.
    \end{align*}
     To apply McDiarmid's inequality to $\tilde{g}$, we need to check that the bounded difference property holds for $\tilde{g}$. Suppose that $i \in [n]$, and $(\vx_1, y_1), \cdots, (\vx_n, y_n), (\vx_i', y_i') \in \mc{X}$ with \[(\mX, \vy) := ((\vx_1, y_1), \cdots, (\vx_n, y_n)) \in \mc{Y}\] and \[(\mX', \vy') := ((\vx_1, y_1), \cdots, (\vx_{i - 1}, y_{i - 1}), (\vx_i', y_i'), (\vx_{i + 1}, y_{i + 1}), \cdots, (\vx_n, y_n)) \in \mc{Y}.\]
     Then
    \begin{align}
    \label{eqn:g-difference-bound}
    \begin{aligned}
        &|\tilde{g}(\mX, \vy) - \tilde{g}(\mX', \vy')|
        \\&\leq \sum_{j \neq i} \Bigg|\min\left(|y_j|, C_2 \sqrt{\log \frac{8n}{\epsilon}}\right)^p \min\left(\delta_j(\mX)^{-\beta}, \left(\frac{4C_1 n^2}{\epsilon}\right)^{\beta/d} \right) \\& - \min\left(|y_j|, C_2 \sqrt{\log \frac{8n}{\epsilon}}\right)^p \min\left(\delta_j(\mX')^{-\beta}, \left(\frac{4C_1 n^2}{\epsilon}\right)^{\beta/d} \right) \Bigg|\\
        &+ \Bigg|\min\left(|y_i|, C_2 \sqrt{\log \frac{8n}{\epsilon}}\right)^p \min\left(\delta_i(\mX)^{-\beta}, \left(\frac{4C_1 n^2}{\epsilon}\right)^{\beta/d}\right)
        \\&- \min\left(|y_i'|, C_2 \sqrt{\log \frac{8n}{\epsilon}}\right)^p \min\left(\delta_i(\mX')^{-\beta}, \left(\frac{4C_1 n^2}{\epsilon}\right)^{\beta/d}\right)
        \Bigg|.
    \end{aligned}
    \end{align}
    The above expression consists of two quantities: a sum over indices $j \neq i$, and the $i$ term. We bound these two quantities separately. Let $\mc{S} \subset [n]$ denote the subset of indices $j \neq i$ for which $\delta_j(\mX) \neq \delta_j(\mX')$.
    By Lemma~\ref{lem:not-too-many-deltas-change}, there exists a constant $C_3 > 1$ depending only on $d$ such that $|\mc{S}| \leq C_3$. Then
    \begin{align}\label{eqn:yj-deltaj-sparsity-bound}
    \begin{aligned}
        &\sum_{j \neq i} \Bigg|\min\left(|y_j|, C_2 \sqrt{\log \frac{8n}{\epsilon}}\right)^p \min\left(\delta_j(\mX)^{-\beta}, \left(\frac{4C_1 n^2}{\epsilon}\right)^{\beta/d} \right) \\& - \min\left(|y_j|, C_2 \sqrt{\log \frac{8n}{\epsilon}}\right)^p \min\left(\delta_j(\mX')^{-\beta}, \left(\frac{4C_1 n^2}{\epsilon}\right)^{\beta/d} \right) \Bigg|\\
        &= \sum_{j \in \mc{S}}\min\left(|y_j|, C_2 \sqrt{\log \frac{8n}{\epsilon}}\right)^p\Bigg|\min\left(\delta_j(\mX)^{-\beta}, \left(\frac{4C_1 n^2}{\epsilon}\right)^{\beta/d} \right) \\&- \min\left(\delta_j(\mX')^{-\beta}, \left(\frac{4C_1 n^2}{\epsilon}\right)^{\beta/d} \right)\Bigg|\\
        &\leq 2\sum_{j \in \mc{S}} \min\left(|y_j|, C_2 \sqrt{\log \frac{8n}{\epsilon}}\right)^p \left(\frac{4 C_1 n^2}{\epsilon} \right)^{\beta/d}\\
        &\leq 2 \sum_{j \in \mc{S}} \left(C_2 \sqrt{\log \frac{8n}{\epsilon}}\right)^p\left(\frac{4 C_1 n^2}{\epsilon} \right)^{\beta/d}\\
        &\leq 2 C_3\left(C_2 \sqrt{\log \frac{8n}{\epsilon}}\right)^p\left(\frac{4 C_1 n^2}{\epsilon} \right)^{\beta/d}.
    \end{aligned}
    \end{align}
    To bound the second term, we observe that it is a difference of bounded expressions:
    \begin{align}
    \begin{aligned}\label{eqn:remaining-term-sparsity-bound}
        &\Bigg|\min\left(|y_i|, C_2 \sqrt{\log \frac{8n}{\epsilon}}\right)^p \min\left(\delta_i(\mX)^{-\beta}, \left(\frac{4C_1 n^2}{\epsilon}\right)^{\beta/d}\right)
        \\&- \min\left(|y_i'|, C_2 \sqrt{\log \frac{8n}{\epsilon}}\right)^p \min\left(\delta_i(\mX')^{-\beta}, \left(\frac{4C_1 n^2}{\epsilon}\right)^{\beta/d}\right)
        \Bigg|\\
        &\leq 2 \left(C_2\sqrt{\log \frac{8n}{\epsilon}}\right)^p\left(\frac{4C_1 n^2}{\epsilon}\right)^{\beta/d}\\
        &\leq 2C_3 \left(C_2\sqrt{\log \frac{8n}{\epsilon}}\right)^p\left(\frac{4C_1 n^2}{\epsilon}\right)^{\beta/d}.
    \end{aligned}
    \end{align}
    Substituting (\ref{eqn:yj-deltaj-sparsity-bound}) and (\ref{eqn:remaining-term-sparsity-bound}) into (\ref{eqn:g-difference-bound}), we get
    \begin{align*}
        |\tilde{g}(\mX, \vy) - \tilde{g}(\mX', \vy')| &\leq 4C_3 \left(C_2\sqrt{\log \frac{8n}{\epsilon}}\right)^p\left(\frac{4C_1 n^2}{\epsilon}\right)^{\beta/d}\\
        &= 4C_3 C_2^p \left(\frac{\log 8n}{\epsilon} \right)^p \left(\frac{4C_1 n^2}{\epsilon}\right)^{\beta/d}.
    \end{align*}
    So $\tilde{g}$ satisfies the bounded difference property. We denote
    \[\alpha = 4 C_3 C_2^p \left(\frac{4 C_1 n^2}{\epsilon}\right)^{\beta/d}\left(\log \frac{8n}{\epsilon}\right)^{p/2}. \]
    
    By Theorem~\ref{thm:mcdiarmid}, for all $t > 0$,
    \begin{align*}
        \P\left(\tilde{g}(\mX, \vy) - \E[\tilde{g}(\mX, \vy) \mid (\mX, \vy) \in \mc{Y}] \geq t  \right) &\leq \exp\left(-\frac{2t^2}{n \alpha^2} \right).
    \end{align*}
    Note that since $\P((\mX, \vy) \in \mc{Y}) = 1$,
    \[\E[\tilde{g}(\mX, \vy) \mid (\mX, \vy) \in \mc{Y}] = \E[\tilde{g}(\mX, \vy)]. \]
    
    Setting $t = \alpha\left(\frac{n}{2} \log \frac{2}{\epsilon}\right)^{1/2}$ yields that with probability at least $1 - \frac{\epsilon}{2}$,
    \begin{align}\label{eqn:tilde-g-expectation-upper-bound}
        \tilde{g}(\mX, \vy) &\leq \E[\tilde{g}(\mX, \vy)] + \alpha\left(\frac{n}{2} \log \frac{2}{\epsilon}\right)^{1/2}.
    \end{align}
    We denote the event that this occurs by $\omega_2$, and will further bound the quantity $\E[\tilde{g}(\mX, \vy)]$ in (\ref{eqn:tilde-g-expectation-upper-bound}). By definition, $\tilde{g}(\mX, \vy) \leq g(\mX, \vy)$ almost surely, so
    \begin{align*}
        \E[\tilde{g}(\mX, \vy)] &\leq \E[g(\mX, \vy)]\\
        &= \sum_{i = 1}^n \E[|y_i|^p\delta_i(\mX)^{-\beta}]\\
        &= n \E[|y_1|^p \delta_1(\mX)^{-\beta}]\\
        &= n \E[\E[|y_1|^p \delta_1(\mX)^{-\beta} \mid \mX ] ]\\
        &= n \E[\delta_1(\mX)^{-\beta} \E[|y_1|^p \mid \mX] ].
    \end{align*}
    Since $y_1$ is independent of $\mX$ conditional on $\vx_1$,
    \begin{align*}
        n \E[\delta_1(\mX)^{-\beta} \E[|y_1|^p \mid \mX] ] &= n \E[\delta_1(\mX)^{-\beta}\E[|y_1|^p \mid \vx_1] ] .
    \end{align*}
    Since $y_1$ is sub-Gaussian conditional on $\vx_1$ (Assumption~\ref{assump:subgaussian-outputs}), $\E[|y_1|^p \mid \vx_1] \lesssim p^{p/2} \lesssim_p 1$ almost surely (see Proposition 2.6.1 of \cite{vershynin_2018}), so
    \begin{align*}
        n \E[\delta_1(\mX)^{-\beta}\E[|y_1|^p \mid \vx_1] ] &\lesssim_{p} n \E[\delta_1(\mX)^{-\beta} ]\\
        &= n \int_0^{\infty}\P(\delta_1^{-\beta} > t)dt\\
        &\leq n \int_0^{n^{\beta/d}} 1 dt + n \int_{n^{\beta/d}}^{\infty}\P(\delta_1^{-\beta} > t)dt\\
        &= n^{\beta/d + 1} + n \int_{n^{\beta/d}}^{\infty}\P(\delta_i < t^{-1/\beta})dt\\
        &\lesssim_{d, \beta} n^{\beta/d + 1} + n \int_{n^{\beta/d}}^{\infty} n t^{-d/\beta}dt & \text{(By (\ref{eqn:delta-beta-cdf}))}\\
        &\lesssim_{d, \beta} n^{\beta/d + 1} + n^2\left(n^{\beta/d}\right)^{-d/\beta + 1}\\
        &\lesssim n^{\beta/d + 1}.
    \end{align*}
    Hence, there exists a constant $C_4 > 0$ depending on $d, p, \beta$ such that
    \begin{align*}
        \E[\tilde{g}(\mX, \vy)] \leq C_4 n^{\beta/d + 1}.
    \end{align*}
    If both $\omega_1$ and $\omega_2$ occur (which happens with probability at least $1- \epsilon$), then substituting the above inequality into \ref{eqn:tilde-g-expectation-upper-bound} yields
    \begin{align*}
        \sum_{i = 1}^n |y_i|^p \delta_i(\mX)^{-\beta} &= g(\mX, \vy) \\&\leq \tilde{g}(\mX, \vy) \\&\leq C_4 n^{\beta/d + 1} + \alpha\left(\frac{n}{2}\log \frac{2}{\epsilon}\right)^{1/2}.
    \end{align*}
    Now if
    \begin{align*}
        n &\geq (\epsilon^{-1})^{(3/4 + \beta/(2d))/(1/4 - \beta/(2d)) },
    \end{align*}
    then
    \begin{align*}
        \alpha\left(\frac{n}{2}\log \frac{2}{\epsilon}\right)^{1/2} &\lesssim_{\beta, d, p} n^{2\beta/d + 1/2}\left(\log \frac{n}{\epsilon} \right)^{p/2}\left(\frac{1}{\epsilon}\right)^{\beta/d}\left(\log  \frac{1}{\epsilon}\right)^{1/2}\\
        &\lesssim_{\beta, d, p} n^{2\beta/d + 1/2}\left(\frac{n}{\epsilon} \right)^{1/4 - \beta/(2d)}\left(\frac{1}{\epsilon}\right)^{\beta/d}\left(\frac{1}{\epsilon}\right)^{1/2} & \text{(Since $\beta < d/2$)}\\
        &= n^{(3 \beta)/(2d) + 3/4}(\epsilon^{-1})^{3/4 + \beta/(2d)}\\
        &\leq n^{(3\beta)/(2d) + 3/4}n^{1/4 - \beta/(2d)}\\
        &= n^{\beta/d + 1}
    \end{align*}
    and therefore
    \begin{align*}
        \sum_{i = 1}^n |y_i|^p \delta_i(\mX)^{-\beta} &\lesssim_{d, p, \beta} n^{\beta/d + 1}.
    \end{align*}
    This happens with probability at least $1 - \epsilon$, which establishes the result.
\end{proof}

Corollary~\ref{corr:min-norm-bound} then follows trivially.

\begin{proof}[Proof of Corollary~\ref{corr:min-norm-bound}]
    This follows from Lemma~\ref{lem:scale-of-delta-concentrates} and Lemma~\ref{lem:min-norm-delta-bound} with $\beta = d-kp$.
\end{proof}

To finish this section, we prove Lemma~\ref{lem:delta-packing-size}, as it uses similar techniques to the proof of Lemma~\ref{lem:scale-of-delta-concentrates}.

\begin{proof}[Proof of Lemma~\ref{lem:delta-packing-size}]
    We follow the proof technique of Lemma~\ref{lem:scale-of-delta-concentrates}, using that the $\delta_i$s are stable under changes of a single data point. Since the density $\rho_{\vx}$ is bounded above and below by constants (Assumption~\ref{assump:data-regularity}), there exists a constant $C_1 \in (0, \infty)$ depending only on $d$ such that
    \begin{align}\label{eqn:delta-ball-volume}
    \P(\|\vx - \vx'\| < \delta) \leq C_1 \delta^{d}\end{align}
for all $\delta > 0$ and all $\vx, \vx'$ drawn iid from the data distribution $\mu_{\vx}$. For $i \in [n]$, we define the following Bernoulli random variables. Let $Z_i$ take the value 1 if $\delta_i \geq \left(\frac{1}{2 C_1 n} \right)^{1/d}$ and the value 0 otherwise. Let $Y_i$ take the value 1 if
\[|y_i| \leq  C_y \sqrt{\log \frac{4}{\rho} } \]
(where $C_y$ is as defined in Assumption~\ref{assump:subgaussian-outputs}), and the value 0 otherwise. Let $W_i$ take the value 1 if
\[\mc{L}(y_i; \vx_i) \geq \sigma + \inf_{\hat{y} \in \R} \mc{L}(\hat{y}; \vx_i), \]
and the value 0 otherwise.
Let $\mc{B} \subset [n]$ be the subset of all indices such that $Z_iY_iW_i = 1$. We show that $\mc{B}$ satisfies the desired conditions with high probability. Since $Z_i = 1$ for all $i \in \mc{B}$, $\mc{B}$ satisfies condition 2. Since $Y_i = 1$ for all $i \in \mc{B}$, $\mc{B}$ satisfies condition 3. Since $W_i = 1$ for all $i \in \mc{B}$,
\begin{align*}
    \mc{L}(y_i; \vx_i) \geq \sigma + \inf_{\hat{y} \in \R} \mc{L}(\hat{y}; \vx_i) = \sigma + \mc{L}(f_{\bayes}(\vx_i); \vx_i),
\end{align*}
where the equality is an application of Lemma~\ref{lem:bayes-optimal-alternate}. So $\mc{B}$ satisfies condition 4. It remains to show that $\mc{B}$ satisfies condition 1 with high probability. To this end, we show that the sum
\[|\mc{B}| =  \sum_{i = 1}^n Z_iY_iW_i \]
concentrates, and start by analyzing the expectation of the terms. By Assumption~\ref{assump:subgaussian-outputs},
\begin{align}\label{eqn:many-terms-small-yi}
\begin{aligned}
    \E[Y_i \mid \mX] &= \P\left(|y_i| \leq 2C_y \sqrt{\log \frac{1}{\rho} } \middle| \mX \right)\\
    &= \P\left(|y_i| \leq 2C_y \sqrt{\log \frac{1}{\rho} } \middle| \vx_i \right)\\
    &\geq 1 - \frac{\rho}{2}
\end{aligned}
\end{align}
almost surely. By Assumption~\ref{assump:label-noise},
\begin{align}\label{eqn:E-Wi-given-X}
    \begin{aligned}
    \E[W_i \mid \mX] &= \P\left(\mc{L}(y_i; \vx_i) \geq \sigma + \inf_{\hat{y} \in \R}\mc{L}(\hat{y}; \vx_i) \middle| \mX \right)\\
    &= \P\left(\mc{L}(y_i; \vx_i) \geq \sigma + \inf_{\hat{y} \in \R}\mc{L}(\hat{y}; \vx_i) \middle| \vx_i \right)\\
    &\geq \rho
    \end{aligned}
\end{align}
almost surely. Combining (\ref{eqn:many-terms-small-yi}) and (\ref{eqn:E-Wi-given-X}), we get
\begin{align}\label{eqn:E-YiWi-given-X}
    \begin{aligned}
        \E[Y_iW_i \mid \mX] &= \P(Y_i = W_i = 1 \mid \mX)\\
        &\geq 1 - \P(Y_i = 0 \mid \mX) - \P(W_i = 0 \mid \mX)\\
        &= \E[Y_i \mid \mX] + \E[W_i \mid \mX] - 1\\
        &\geq \frac{\rho}{2}
    \end{aligned}
\end{align}
almost surely.
Then
\begin{align*}
    \E[Z_iY_iW_i] &= \E[\E[Z_iY_iW_i \mid \mX]]\\
    &= \E[Z_i \E[Y_i W_i \mid \mX]]& \text{(Since $Z_i$ is $\mX$-measurable)}\\
    &\geq \frac{\rho}{2} \E[Z_i] & \text{(By \ref{eqn:E-YiWi-given-X})}\\
    &= \frac{\rho}{2} \P(\delta_i \geq ( 2C_1 n)^{-1/d})\\
    &= \frac{\rho}{2}\left(1 -  \P(\exists j \neq i: \|\vx_i - \vx_j\| < (2C_1 n)^{-1/d})\right)\\
    &\geq \frac{\rho}{2} \left(1 - \sum_{j \neq i}\P(\|\vx_i - \vx_j\| < (2C_1 n)^{-1/d}) \right)\\
    &\geq \frac{\rho}{2} \left(1 - \sum_{j \neq i} C_1(2C_1 n)^{-d/d} \right) & \text{(By (\ref{eqn:delta-ball-volume}) )}\\
    &\geq \frac{\rho}{4}.
\end{align*}
The sum
\[\sum_{i = 1}^n Z_i Y_iW_i \]
is a function of the dataset $(\mX, \vy)$. If we change one of the points $(\vx_j, y_j)$ to a different value $(\vx_j', y_j')$, then at most $C_2$ of the values $Z_i$ will change in value by Lemma~\ref{lem:not-too-many-deltas-change} (where $C_2$ is a constant depending only on $d$), and only one of the values $Y_i$ and $W_i$ (the $j$-th value) will change. So changing one of the points can alter the value of the sum by at most $C_2 + 1$. By McDiarmid's inequality, for all $t \geq 0$,
\begin{align*}
    \P\left(\sum_{i = 1}^n Z_iY_iW_i \leq \frac{\rho n}{4} - t \right) &\leq\P\left(\sum_{i = 1}^n Z_iY_iW_i \leq \sum_{i = 1}^n \E[Z_iY_iW_i] - t \right) \\&\leq \exp\left(-\frac{2t^2}{n (C_2 + 1)^2}\right).
\end{align*}
Setting $t = \frac{\rho n}{8}$ with $n \geq \frac{32(C_2 + 1)^2}{\rho^2} \log \frac{1}{\epsilon}$, we obtain with probability at least $1 - \epsilon$ that
\begin{align*}
    |\mc{B}| = \sum_{i = 1}^n Z_i Y_i \geq \frac{\rho n}{8}. 
\end{align*}
So $\mc{B}$ also satisfies condition 1 with probability at least $1 - \epsilon$.
\end{proof}

\subsection{Proofs for Sobolev inequalities}

To prove Corollary~\ref{corr:morrey-taylor-specific}, we first establish some useful Sobolev inequalities for balls in $\R^d$ (see Section 5.6.2 in \citet{evans2022partial}).
\begin{lemma}[Morrey's inequality]\label{lem:morrey}
    Let $\delta > 0$ and suppose that $u \in C^1(\R^d)$. Then for all $\vx_0 \in \R^d$ and all $\vx_1 \in B(\vx_0, \delta)$,
    \[|u(\vx_1) - u(\vx_0)| \lesssim_{d, p} \delta^{1 - d/p} \left(\int_{B(\vx_0, 2\delta)} \|D u(\vx)\|^p d\vx \right)^{1/p}. \]
\end{lemma}

\begin{corollary}\label{corr:morrey-taylor}
   Let $\delta > 0$ and $d < kp$, and suppose that $u \in C^{\infty}(\R^d)$ with $\|u\|_{W^{k, p}(\R^d) } < \infty$. Then for all $j \in \{0, 1, \cdots, k - 1\}$, $\vx_0 \in \R^d$ and all $\vx_1 \in B(\vx_0, \delta)$,
   $$\|D^ju(\vx_1) - D^ju(\vx_0)\|
     {}\lesssim_{k, d, p} \delta^{k - j - d/p} \left(\int_{B(\vx_0, 2\delta)} \|D^k u(\vx)\|^p d\vx \right)^{1/p}.$$
\end{corollary}
\begin{proof}[Proof of Corollary~\ref{corr:morrey-taylor}]
   We prove by reverse induction on $j$. The base case $j = k - 1$ follows from Lemma~\ref{lem:morrey}. Now suppose that the statement holds for some $j \in \{1, \cdots, k - 1\}$. Then for some constant $C_{k, d, p} > 0$,
    \begin{align*}
        \|D^j u(\vx_1) - D^j u(\vx_0)\| &\leq C_{k, d, p} \delta^{ k - j - d/p} \left(\int_{B(\vx_0, 2\delta)} \|D^k u(\vx)\|^p d\vx \right)^{1/p}
    \end{align*}
    for all $\vx_0 \in \R^d$ and $\vx_1 \in B(\vx_0, \delta)$.
    Now suppose that $\vx_0, \vx_1 \in \R^d$ with $\|\vx_0 - \vx_1\| < \delta$. By the fundamental theorem of calculus,
    \begin{align*}
        &\|D^{j - 1} u(\vx_1) - D^{j - 1} u(\vx_0)\|\\
        &=\left\|\left[ D^{j-1} u(\vx_1) - D^{j-1} u\left(\frac{1}{2}(\vx_0 + \vx_1)\right) \right] - \left[ D^{j-1} u(\vx_0) - D^{j-1} u\left(\frac{1}{2}(\vx_0 + \vx_1)\right) \right]\right\|\\
        &= \left\|\int_0^1 2 \frac{d}{dt}\left[ D^{j - 1} u\left(\vx_0 + \frac{t+1}{2}(\vx_1 - \vx_0) \right)- D^{j - 1} u\left(\vx_1 + \frac{t+1}{2}(\vx_0 - \vx_1) \right) \right] dt \right\| \\
        &\leq 2 \int_0^1 \left\|D^j u\left(\vx_0 + \frac{t+1}{2}(\vx_1 - \vx_0) \right) - D^j u\left(\vx_1 + \frac{t+1}{2}(\vx_0 - \vx_1)\right) \right\| \|\vx_1 - \vx_0\| dt\\
        &\leq 2 \delta \sup_{t \in [0, 1]} \left\|D^j u\left(\vx_0 + \frac{t+1}{2}(\vx_1 - \vx_0) \right) - D^j u\left(\vx_1 + \frac{t+1}{2}(\vx_0 - \vx_1)\right)\right\|\\
        &\leq C_{k, d, p} \delta^{1 + k - j - d/p} \left(\int_{B(\vx_0, 2\delta)} \|D^k u(\vx)\|^p d\vx \right)^{1/p},
    \end{align*}
    where in the last line we used the inductive hypothesis and that
\[\left\|\left(\vx_0 + \frac{t+1}{2}(\vx_1 - \vx_0)\right) - \left(\vx_1 + \frac{t+1}2\vx_0 - \vx_1)\right)\right\| = \|t(\vx_1 - \vx_0)\| < \delta\text{ for }t \in [0,1].\]
So we have shown that the result holds for $j - 1$, and therefore the result holds for all $j \in \{0, 1, \cdots, k - 1\}$ by induction.
\end{proof}

Corollary~\ref{corr:morrey-taylor-specific} then follows using some additional results from Sobolev space theory.

\begin{proof}[Proof of Corollary~\ref{corr:morrey-taylor-specific}]
Let $U = B(\vx_0, 4\delta)$ and let $V = B(\vx_0, 2 \delta)$. By the Sobolev embedding theorem (Theorem~\ref{thm:sobolev-embedding}), %there exists a continuous function $u^* \in C^0(\R^d)$ such that $u^* = u$ almost everywhere so that 
pointwise evaluation is well-defined and continuous on $W^{k, p}(U)$. Now by Theorem~5.3.3 in \cite{evans2022partial}, there exist functions $u_m \in C^{\infty}(\overline{U})$ converging to $u$ in $W^{k, p}(U)$. For each $m$, let $v_m \in C^{\infty}(\R^d)$ be a smooth function which is equal to $u_m$ on $V$. Since $u_m \to u$ in $W^{k, p}(U)$, we also have $u_m \to u$ in $W^{k, p}(V)$, and so $v_m \to u$ in $W^{k, p}(V)$.  %Furthermore, by the Rellich-Kondrachov Compactness Theorem, up to a subsequence these functions converge uniformly on the compact subset $B(\vx_0, 2\delta)$.
Then by Corollary~\ref{corr:morrey-taylor} we have
\begin{align*}
    |u(x_1) - u(x_0)| &= \lim_{m \to \infty} |v_m(x_1) - v_m(x_0)|\\
    &\lesssim_{k, d, p} \lim_{m \to \infty} \delta^{k - d/p} \left(\int_{B(\vx_0, 2\delta)} \|D^k v_m(\vx)\|^p d\vx \right)^{1/p}\\
    &=  \delta^{k - d/p} \left(\int_{B(\vx_0, 2\delta)} \|D^k u(\vx)\|^p d\vx \right)^{1/p}\\
    &\leq \delta^{k - d/p} \|u\|_{W^{k,p}(\B(\vx_0, 2\delta))},
\end{align*}
where the second to last line follows as $v_m \to u$ in $W^{k,p}(V)$ implies that $D^kv_m \to D^ku$ in $L^p(V)$. The result follows by raising both sides to the $p$th power.
\end{proof}

The above corollary shows that for an interpolating function $f^*$ and a corrupt point $i \in \mc{B}$, we can find a ball around $\vx_i$ within which the function does not vary too much. In order for us to use this,  a nontrivial fraction of this ball should be contained in $\Omega$. The following two theorems establish this.

Let $\Omega$ be a bounded open subset of $\R^d$. We say that $\Omega$ is a \emph{$W^{k, p}$-extension domain} if there exists a bounded linear operator $E: W^{k, p}(\Omega) \to W^{k, p}(\R^d)$ such that for all $f \in W^{k, p}(\Omega)$, $Ef|_{\Omega} = f$.
\begin{theorem}[Extension theorem]\label{thm:extension}
    If $\Omega$ is a bounded open subset of $\R^d$ with $C^1$ boundary, then $\Omega$ is a $W^{1, p}$-extension domain for all $p \in [1, \infty]$.
\end{theorem}

For a proof of the above theorem, see \citet[Section 5.4]{evans2022partial}.
\begin{theorem}[\citealt{koskela1990capacity}, Theorem 6.5]
\label{thm:koskela-intersection}
    Let $\Omega \subset \R^d$ be a $W^{1, p}$-extension domain for some $p > d - 1$. Then 
    \[|\Omega \cap B(\vx_0, \delta)| \gtrsim_{d, \Omega} |B(\vx_0, \delta)|\]
    for all $\vx_0 \in \overline{\Omega}$ and $\delta \in [0, \diam(\Omega)]$. 
\end{theorem}

%%%%%%%%%%%%%%%%%%%%%%%%%%%%%%%%%%%%%%%%%%%%%%%%%%%%%%%%%%%%%%%%%%%%%%%%%%%%%%%%%%%%%%%%%%%%%%%%%%%%%%%%%%%%%%%%%%%%%%%%%%%%%%%%%%%%%%%%%%%%%%%%%%%%%%%%%%%%%%%%%%%%%%%%%%%%%%%%%%%%%%%%%%%%%%%%%%%%%%%%%%%%%%%%%%%%%%%%%%%%%%%%%%%%%%%%%%%%%%%%%%%%%%%%%%%%%%%%%%%%%%%%%%%%%%%%%%%%%%%%%%%%%%%%%%%%%%%%%%%%%%%%%%%%%%%%%%%%%%%%%%%%%%%%%%%%%%%%%%%%%%%%%%%%%%%%%%%%%%%%%%%%%%%%%%%%%%%%%%%%%%%%%%%%%%%%%%%%%%%%%%%%%%%%%%%%%%%%%%%%%%%%%%%%%%%%%%%%%%%%%%%%%%%%%%%%%%%%%%%%%%%%%

\end{document}